\documentclass[journal]{IEEEtran}
\usepackage{amsmath,amssymb,epsfig,verbatim}
\usepackage[tight]{subfigure}
\usepackage{ctable}

\usepackage{float}

\def\reals{\mathbb{R}}

\def\comp{\raise 1pt \hbox{$\scriptstyle\circ$}}

\def\argmin{\mathop{\rm argmin}\limits}
\def\minimize{\mathop{\rm minimize}\limits}

\def\st{\mathop{\rm subject\ to}}

\def\dom{\mathop{\rm dom}}
\def\gph{\mathop{\rm gph}}

\def\rge{\mathop{\rm rge}}

\def\upto{{\raise 1pt \hbox{$\scriptstyle \,\nearrow\,$}}}
\def\downto{{\raise 1pt \hbox{$\scriptstyle \,\searrow\,$}}}

\def\tos{\rightrightarrows}

\newtheorem{theorem}{Theorem}

\newtheorem{corollary}[theorem]{Corollary}
\newtheorem{definition}[theorem]{Definition}
\newtheorem{example}[theorem]{Example}

\newenvironment{proof}
{\begin{trivlist}\item[\, 
{\bf Proof.}]}{{\hfill $\square$}\end{trivlist}}

\begin{document}
\newcommand{\bolda}{\mbox{$\mathbf{a}$}}
\newcommand{\boldb}{\mbox{$\mathbf{b}$}}
\newcommand{\boldx}{\mbox{$\mathbf{x}$}}
\newcommand{\boldy}{\mbox{$\mathbf{y}$}}
\newcommand{\boldf}{\mbox{$\mathbf{f}$}}
\newcommand{\boldz}{\mbox{$\mathbf{z}$}}
\newcommand{\boldF}{\mbox{$\mathbf{F}$}}
\newcommand{\boldG}{\mbox{$\mathbf{G}$}}
\newcommand{\boldg}{\mbox{$\mathbf{g}$}}
\newcommand{\boldh}{\mbox{$\mathbf{h}$}}
\newcommand{\boldH}{\mbox{$\mathbf{H}$}}
\newcommand{\boldzero}{\mbox{$\mathbf{0}$}}
\newcommand{\Rbb}{\mbox{$\mathbb R$}}

\title{Efficient Evolutionary Algorithm for Single-Objective Bilevel Optimization}

\author{Ankur Sinha\thanks{Ankur Sinha is a researcher at Dept. of Information and Service Economy, Aalto University School of Business, Helsinki, Finland {\tt ankur.sinha@aalto.fi}}, Pekka Malo\thanks{Pekka Malo is an Assistant Professor at Dept. of Information and Service Economy, Aalto University School of Business, Helsinki, Finland {\tt pekka.malo@aalto.fi}}, and Kalyanmoy Deb\thanks{Kalyanmoy Deb is a Professor at Dept. of Electrical and Computer Engineering, Michigan State University, East Lansing, Michigan {\tt kdeb@egr.msu.edu}}}


\maketitle

\begin{abstract}
Bilevel optimization problems are a class of challenging optimization problems, which contain two levels of optimization tasks. In these problems, the optimal solutions to the lower level problem become possible feasible candidates to the upper level problem. Such a requirement makes the optimization problem difficult to solve, and has kept the researchers busy towards devising methodologies, which can efficiently handle the problem. Despite the efforts, there hardly exists any effective methodology, which is capable of handling a complex bilevel problem. In this paper, we introduce bilevel evolutionary algorithm based on quadratic approximations (BLEAQ) of optimal lower level variables with respect to the upper level variables. The approach is capable of handling bilevel problems with different kinds of complexities in relatively smaller number of function evaluations. Ideas from classical optimization have been hybridized with evolutionary methods to generate an efficient optimization algorithm for generic bilevel problems. The performance of the algorithm has been evaluated on two sets of test problems. The first set is a recently proposed SMD test set, which contains problems with controllable complexities, and the second set contains standard test problems collected from the literature. The proposed method has been compared against three benchmarks, and the performance gain is observed to be significant.
\end{abstract}


\IEEEdisplaynontitleabstractindextext

\IEEEpeerreviewmaketitle

\section{Introduction}
Bilevel optimization is a branch of optimization, which contains a nested optimization problem within the constraints of the outer optimization problem. The outer optimization task is usually referred as the upper level task, and the nested inner optimization task is referred as the lower level task. The lower level problem appears as a constraint, such that only an optimal solution to the lower level optimization problem is a possible feasible candidate to the upper level optimization problem. Such a requirement makes bilevel optimization problems difficult to handle and have kept researchers and practitioners busy alike. The hierarchical optimization structure may introduce difficulties such as non-convexity and disconnectedness even for simpler instances of bilevel optimization like bilevel linear programming problems. Bilevel linear programming is known to be strongly NP-hard \cite{hansen92}, and it has been proven that merely evaluating a solution for optimality is also a NP-hard task \cite{vicente94}. This gives us an idea about the kind of challenges offered by bilevel problems with complex (non-linear, non-convex, discontinuous etc.) objective and constraint functions.

In the field of classical optimization, a number of studies have been conducted on bilevel programming \cite{colson,vicente-review,dempe-dutta}. Approximate solution techniques are commonly employed to handle bilevel problems with simplifying assumptions like smoothness, linearity or convexity. Some of the classical approaches commonly used to handle bilevel problems include the Karush-Kuhn-Tucker approach \cite{bianco-kkt,bilevel-KKT1}, Branch-and-bound techniques \cite{bard82}, and the use of penalty functions \cite{aiyoshi81}. Most of these solution methodologies are rendered inapplicable, as soon as the bilevel optimization problem becomes complex. Heuristic procedures such as evolutionary algorithms have also been developed for handling bilevel problems with higher levels of complexity \cite{yin-bilevel,GA_Wang}. Most of the existing evolutionary procedures often involve enormous computational expense, which limits their utility to solving bilevel optimization problems with smaller number of variables.

There are a number of practical problems which are bilevel in nature. They are often encountered in transportation (network design, optimal pricing) \cite{migdalas95,constantin95,brotcorne01}, economics (Stackelberg games, principal-agent problem, taxation, policy decisions) \cite{fudenberg93,stackelbergWang01,my-caor13,my-cec13}, management (network facility location, coordination of multi-divisional firms) \cite{sun08,bard83}, engineering (optimal design, optimal chemical equilibria) \cite{kirjnerneto98,smith82} etc \cite{dempe2003,bard98}. Complex practical problems are usually modified into a simpler single level optimization task, which is solved to arrive at a {\em satisficing}\footnote{Satisficing is a portmanteau of two words, satisfy and suffice. A satisficing solution need not be optimal but meets the needs of a decision maker.} instead of an optimal solution. For the complex bilevel problems, classical methods fail due to real world difficulties like non-linearity, discreteness, non-differentiability, non-convexity etc. Evolutionary methods are not very useful either because of their enormous computational expense. Under such a scenario, a hybrid strategy could be solution. Acknowledging the drawbacks associated with the two approaches, we propose a hybrid strategy that utilizes principles from classical optimization within an evolutionary algorithm to quickly approach a bilevel optimum. The proposed method is a bilevel evolutionary algorithm based on quadratic approximations (BLEAQ) of the lower level optimal variables as a function of upper level variables. 

The remainder of the paper is organized as follows. In the next section, we provide a review of the past work on bilevel optimization using evolutionary algorithms, followed by description of a general bilevel optimization problem. Thereafter, we provide a supporting evidence that a strategy based on iterative quadratic approximations of the lower level optimal variables with respect to the upper level variables could be used to converge towards the bilevel optimal solution. This is followed by the description of the methodology that utilizes the proposed quadratic approximation principle within the evolutionary algorithm. The proposed ideas are further supported by results on a number of test problems. Firstly, BLEAQ is evaluated on recently proposed SMD test problems \cite{my-cec12a}, where the method is shown to successfully handle test problems with 10 decision variables. A performance comparison is performed against a nested evolutionary strategy, and the efficiency gain is provided. Secondly, BLEAQ is evaluated on a set of standard test problems chosen from the literature \cite{shimizu81,aiyoshi84,amouzegar99,bard-book98,outrata90,oduguwa-roy,liu98,wang11}. Most of the standard test problems are constrained problems with relatively smaller number of variables. In order to evaluate the performance of BLEAQ, we choose the algorithms proposed in \cite{wang05,wang11}, which are able to successfully solve all the standard test problems  


\section{Past research on Bilevel Optimization using Evolutionary Algorithms}
Evolutionary algorithms for bilevel optimization have been proposed as early as in the 1990s. One of the first evolutionary algorithm for handling bilevel optimization problems was proposed by Mathieu et al. \cite{mathieu}. The proposed algorithm was a nested strategy, where the lower level was handled using a linear programming method, and the upper level was solved using a genetic algorithm (GA). Nested strategies are a popular approach to handle bilevel problems, where for every upper level vector a lower level optimization task is executed. However, they are computationally expensive and not feasible for large scale bilevel problems. Another nested approach was proposed in \cite{yin-bilevel}, where the lower level was handled using the Frank-Wolfe algorithm (reduced gradient method). The algorithm was successful in solving non-convex bilevel optimization problems, and the authors claimed it to be better than the classical methods. In 2005, Oduguwa and Roy \cite{oduguwa-roy} proposed a co-evolutionary approach for finding optimal solution for bilevel optimization problems. Their approach utilizes two populations. The first population handles upper level vectors, and the second population handles lower level vectors. The two populations interact with each other to converge towards the optimal solution. An extension of this study can be found in \cite{legillon12}, where the authors solve a bilevel application problem with linear objectives and constraints.

Wang et al. \cite{wang05} proposed an evolutionary algorithm based on a constraint handling scheme, where they successfully solve a number of standard test problems. Their approach finds better solutions for a number of test problems, as compared to what is reported in the literature. The algorithm is able to handle non-differentiability at the upper level objective function. However, the method may not be able to handle non-differentiability in the constraints, or the lower level objective function. Later on, Wang et al. \cite{wang11} provided an improved algorithm that was able to handle non-differentiable upper level objective function and non-convex lower level problem. The algorithm was shown to perform better than the one proposed in \cite{wang05}. Given the robustness of the two approaches \cite{wang05,wang11} in handling a variety of standard test problems, we choose these methods as benchmarks. Another evolutionary algorithm proposed in \cite{li06} utilizes particle swarm optimization to solve bilevel problems. Even this approach is nested as it solves the lower level optimization problem for each upper level vector. The authors show that the approach is able to handle a number of standard test problems with smaller number of variables. However, they do not report the computational expense of the nested procedure. A hybrid approach proposed in \cite{li07b}, which is also nested, utilizes simplex-based crossover strategy at the upper level, and solves the lower level using one of the classical approaches. This method successfully solves a number of standard test problems. However, given the nested nature of the algorithm it is not scalable for large number of variables. The authors report the number of generations and population sizes required by the algorithm that may be used to compute the function evaluations at the upper level, but they do not explicitly report the total number of function evaluations required at the lower level. Other studies where authors have relied on a nested strategy include \cite{my-caor13,angelo13}. In both of these studies an evolutionary algorithm has been used at both levels to handle bilevel problems.

Researchers in the field of evolutionary algorithms have also tried to convert the bilevel optimization problem into a single level optimization problem using the Karush-Kuhn-Tucker (KKT) conditions \cite{GA_Wang,li07a,li11}. However, such conversions are possible only for those bilevel problems, where the lower level is smooth and the KKT conditions can be easily produced. 
Recently, there has also been an interest in multi-objective bilevel optimization using evolutionary algorithms. Some of the studies in the direction of solving multi-objective bilevel optimization problems using evolutionary algorithms are \cite{halter-sanaz,shi-xia,my-ecj10,my-ifac09,ruuska12,zhang12}.

\section{Single-Objective Bilevel Problem}
Bilevel optimization is a nested optimization problem that involves two levels of optimization tasks. The structure of a bilevel optimization problem demands that the optimal solutions to the lower level optimization problem may only be considered as feasible candidates for the upper level optimization problem. 
The problem contains two classes of variables: the upper level variables $x_u\in X_U\subset \reals^n$, and the lower level variables $x_l\in X_L\subset \reals^m$. For the lower level problem, the optimization task is performed with respect to variables $x_l$, and the variables $x_u$ act as parameters. A different $x_u$ leads to a different lower level optimization problem, whose optimal solution needs to be determined. The upper level problem usually involves  all variables $x=(x_u,x_l)$, and the optimization is expected to be performed with respect to both sets of variables. In the following we provide two equivalent definitions of a bilevel optimization problem: 
\begin{definition}\label{def:bilevel1}
For the upper-level objective function $F:\reals^n\times\reals^m \to\reals$ and lower-level objective function $f:\reals^n\times\reals^m \to\reals$ 
\begin{align*}
\minimize_{x_u\in X_U,x_l\in X_L}\quad & F_0(x_u,x_l) \\
\st\quad  & x_l\in \argmin 
	\lbrace
		f_0(x_u,x_l):f_j(x_u,x_l)\leq 0, \\
		& \hspace{32mm} \quad \quad j=1,\dots,J
	\rbrace\\
 & F_k(x_u,x_l)\leq 0, k=1,\dots,K
\end{align*}
\end{definition}

The above definition can be stated in terms of set-valued mappings as follows:

\begin{definition}\label{def:bilevel2}
Let $\Psi:\reals^n\tos\reals^m$ be a set-valued mapping, 
$$
\Psi(x_u)=\argmin\{f_0(x_u,x_l):f_j(x_u,x_l)\leq 0, j=1,\dots,J\},
$$
which represents the constraint defined by the lower-level optimization problem, i.e. $\Psi(x_u)\subset X_L$ for every $x_u\in X_U$. Then the bilevel optimization problem can be expressed as a general constrained optimization problem: 
\begin{align*}
\minimize_{x_u\in X_U, x_l\in X_L}\quad & F_0(x_u,x_l) \\
\st\quad  & x_l \in \Psi(x_u) \\
 & F_k(x_u,x_l)\leq 0, k=1,\dots,K
\end{align*}
where $\Psi$ can be interpreted as a parameterized range-constraint for the lower-level decision $x_l$. 
\end{definition}

The graph of the feasible-decision mapping is interpreted as a subset of $X_U\times X_L$
$$
\gph\Psi=\{(x_u,x_l)\in X_U\times X_L \ | \ x_l\in \Psi(x_u)\},
$$
which displays the connections between the upper-level decisions and corresponding optimal lower-level decisions. The domain of $\Psi$, which is obtained as a projection of $\gph\Psi$ on the upper-level decision space $X_U$ represents all the points $x_u\in X_U$, where the lower-level problem has at least one optimal solution, i.e. 
$$
\dom\Psi=\{x_u|\Psi(x_u)\neq \emptyset\}.
$$
Similarly, the range is given by
$$
\rge\Psi=\{x_l|x_l\in\Psi(x_u) \ \text{for some $x_u$}\},
$$
which corresponds to the projection of $\gph\Psi$ on the lower-level decision space $X_L$.

Figure~\ref{fig:explain2} describes the structure of a bilevel problem in terms of two components: (i) $\gph\Psi$, which gives the graph of the lower level decision-mapping $\Psi$ as a subset of $X_U\times X_L$; and (ii) the plot of $F_0$ evaluated on $\gph\Psi$, which shows the upper level objective function with respect to upper level variables $x_u$, when the lower level is optimal $x_l\in\Psi(x_u)$.
\begin{figure}[t]
\begin{center}
\epsfig{file=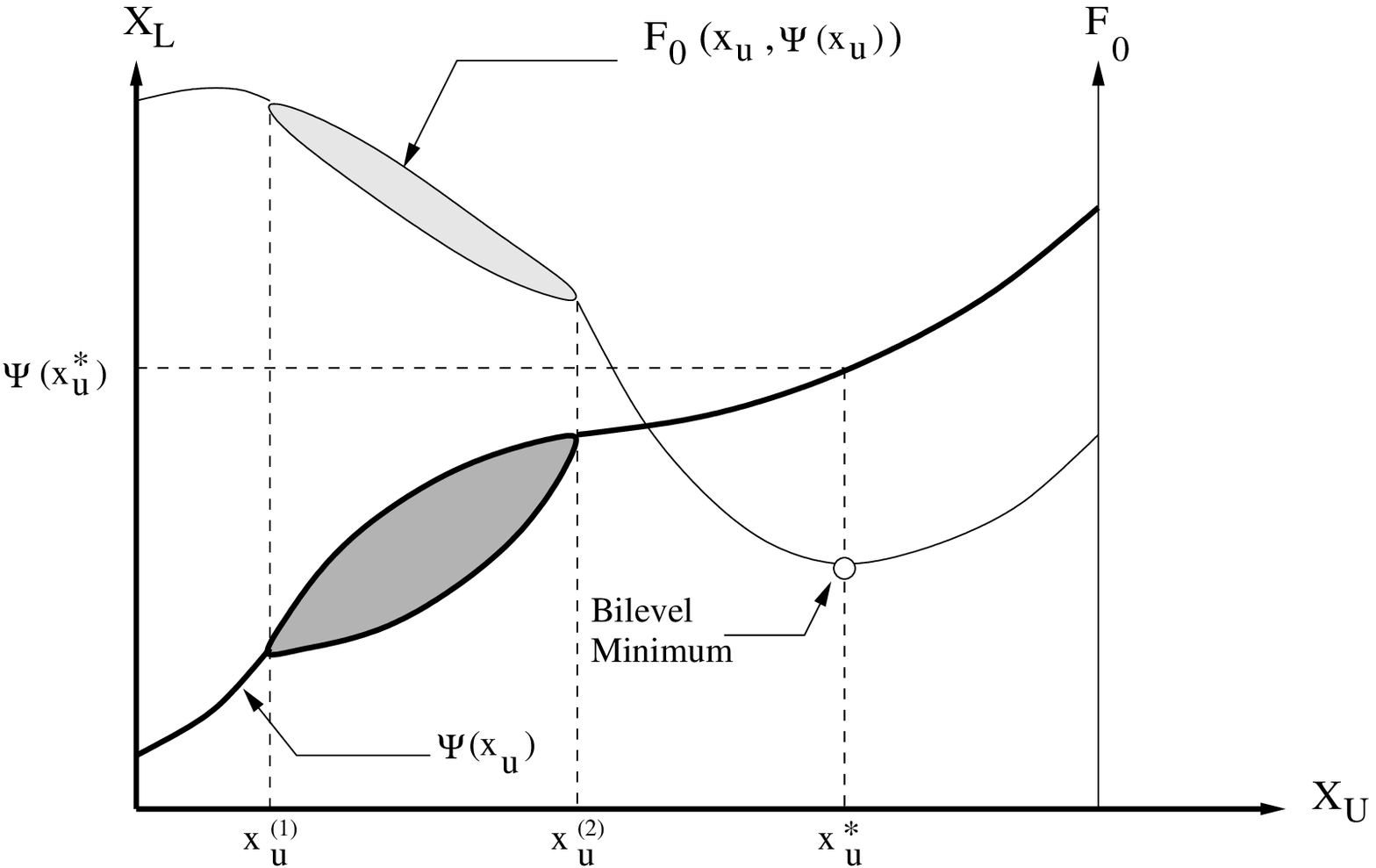,width=0.9\linewidth}
\caption{The $\Psi$-mapping for $x_u$ and optimal $x_l$, and upper level objective function with respect to $x_{u}$ when $x_{l}$ is optimal.}
\label{fig:explain2}
\end{center}
\end{figure}
The shaded area of $\gph\Psi$ shows the regions where there are multiple lower level optimal vectors corresponding to any upper level vector. On the other hand, the non-shaded parts of the graph represent the regions where $\Psi$ is a single-valued mapping, i.e. there is a single lower level optimal vector corresponding to any upper level vector. Considering $\gph\Psi$ as the domain of $F_0$ in the figure, we can interpret $F_0$ entirely as a function of $x_u$, i.e. $F_0(x_u,\Psi(x_u))$. Therefore, whenever $\Psi$ is multi-valued, we can also see a shaded region in the plot of $F_0$ which shows the different upper level function values for any upper level vector with multiple lower level optimal solutions. For instance, in Figure~\ref{fig:explain2}, the shaded region of $\gph\Psi$ corresponds to the shaded region of $F_0$ for the upper level vectors between $x_{u}^{1}$ and $x_{u}^{2}$.

For a more detailed illustration, see the 3-dimensional graph in Figure~\ref{fig:explain4}, where the values of the upper and lower level objective functions $F_0$ and $f_0$ are plotted against the decision space $X_U\times X_L$. 
\begin{figure}[t]
\begin{center}
\epsfig{file=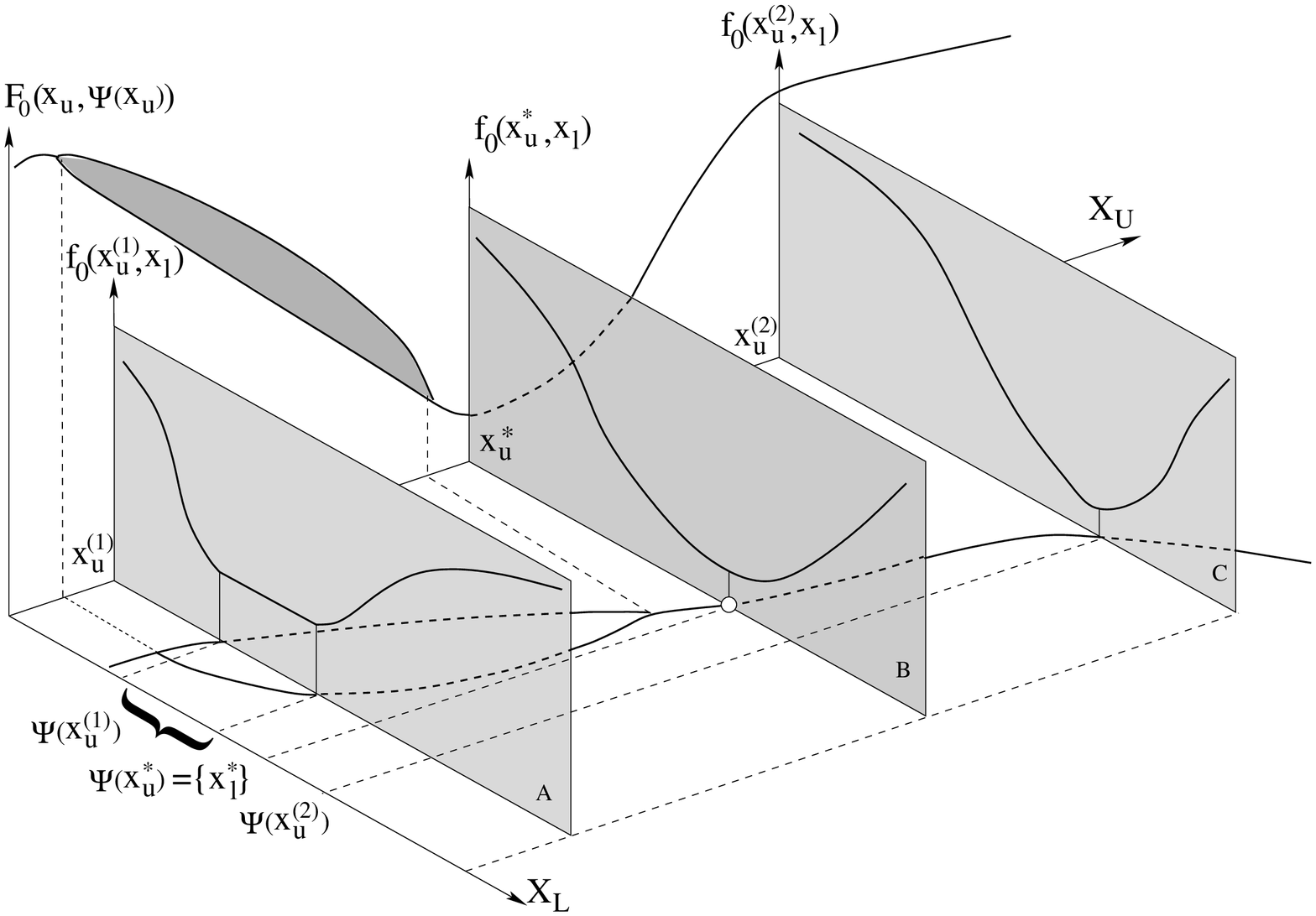,width=0.9\linewidth}
\caption{Graphical representation of a simple bilevel optimization problem.}
\label{fig:explain4}
\vspace{-1mm}
\end{center}
\end{figure}

For simplicity, let us again assume that $\gph\Psi$ is the domain of $F_0$. If we now consider the values of $F_0$ plotted against the upper level decision variables, we obtain a similar information as before. However, as an addition to the previous figure, we have also described how the lower level objective function $f_0$ depends on the upper level decisions. In the figure, the shaded planes marked as A, B and C represent three different lower level optimization problems parameterized by $x_{u}^{(1)}$, $x_{u}^{\star}$ and $x_{u}^{(2)}$, respectively. Once an upper-level decision vector has been fixed, the lower level objective can be interpreted entirely as a function of $x_l$. Hence each shaded plane shows a single-variable plot of $f_0$ against $X_L$ given a fixed $x_u$. 
Consider, for example, the plane A corresponding to the upper level decision $x_{u}^{(1)}$. From the shape of the lower level objective function it is easy to see that there are multiple optimal solutions at the lower level. Therefore, $\Psi$ must also be set-valued at this point, and the collection of optimal lower level solutions is given by $\Psi(x_{u}^{(1)})$. For the other two shaded planes B and C, there is only a single lower level optimal solution for the given $x_u$, which corresponds to $\Psi$ being single-valued at these points. The optimal upper level solution is indicated by point $(x_{u}^{\star},x_{l}^{\star})$, where $\Psi(x_u^{\star})=\{x_{l}^{\star}\}$.

\begin{example}\label{example1}
To provide a further insight into bilevel optimization, we consider a simple example with non-differentiable objective functions at upper and lower levels. The problem has a single upper and lower level variable, and does not contain any constraints at either of the two levels.
\begin{align*}
\minimize_{x_u,x_l}\quad & F_0(x_u,x_l) = |x_u| + x_l - 1\\
\st\quad  & x_l\in \argmin_{x_l} 
	\lbrace
		f_0(x_u,x_l) = (x_{u})^2 + |x_l - e^{x_u}|
	\rbrace\\
\end{align*}
\end{example}
For a given upper level variable $x_u$, the optimum of the lower level problem in the above example is given by $x_l = e^{x_u}$. Therefore, this problem has a single-valued $\Psi$ mapping, such that $\Psi(x_u) = \{e^{x_u}\}$. The $\Psi$ mapping reduces this problem to a simple single-level optimization problem as a function of $x_u$. The optimal solution to the bilevel optimization problem is given by $(x_{u}^{\star},x_{l}^{\star})=(0,1)$. In this example, the lower level optimization problem is non-differentiable at the optimum for any given $x_u$, and the upper level objective function is non-differentiable at the bilevel optimum. Even though the problem involves simple objective functions at both levels, most of the KKT-based methods will face difficulties in handling such a problem. It is noteworthy that even though the objective functions at both levels are non-differentiable, the $\Psi$-mapping is continuous and smooth. For complex examples having disconnected or non-differentiable $\Psi$-mapping one can refer to \cite{dempe02}.

\section{Localization of the Lower Level Problem}
Ideally, the analysis of a bilevel problem would be greatly simplified if the optimal solution mapping $\Psi$ could be treated as if it were an ordinary function. In particular, for the design of an efficient bilevel algorithm, it would be valuable to identify the circumstances under which single-valued functions can be used to construct local approximations for the optimal solution mapping. Given that our study of solution mappings is closely related to sensitivity analysis in parametric optimization, there already exists considerable research on the regularity properties of solution mappings, and especially on the conditions for obtaining single-valued localizations of general set-valued mappings. To formalize the notions of localization and single-valuedness in the context of set-valued mappings, we have adopted the following definition from Dontchev and Rockafellar~\cite{dontchev09}:

\begin{definition}[Localization and single-valuedness]
For a given set-valued mapping $\Psi:X_U\tos X_L$ and a pair $(x_u,x_l)\in\gph{\Psi}$, a {\it graphical localization} of $\Psi$ at $x_u$ for $x_l$ is a set-valued mapping $\Psi_{\rm loc}$ such that 
$$
\gph\Psi_{\rm loc}=(U\times L)\cap\gph \Psi
$$
for some upper-level neighborhood $U$ of $x_u$ and lower-level neighborhood $L$ of $x_l$, i.e. 
\[
\Psi_{\rm loc}(x_u)=\begin{cases}
\Psi(x_u)\cap L & \text{for $x_u\in U$,}\\
\emptyset & \text{otherwise.}
\end{cases}
\]
If $\Psi_{\rm loc}$ is actually a function with domain $U$, it is indicated by referring to a single-valued localization $\psi_{\rm loc}:U\to X_L$ around $x_u$ for $x_l$. 
\end{definition}

Obviously, graphical localizations of the above type can be defined for any lower level decision mapping. However, it is more interesting to ask when is the localization not only a single-valued function but also possesses convenient properties such as continuity and certain degree of smoothness. In this section, our plan is to study how the lower-level solution mappings behave under small perturbations, and clarify the regularity conditions in which the solution mapping can be locally characterized by a Lipschitz-continuous single-valued function. We begin the discussion from simple problems with convex set-constraints in section~\ref{sec:convex-constraints} and gradually extend the results to problems with general non-linear constraints in section~\ref{sec:nonlinear-constraints}.  

\begin{figure}[t]
\begin{center}
\epsfig{file=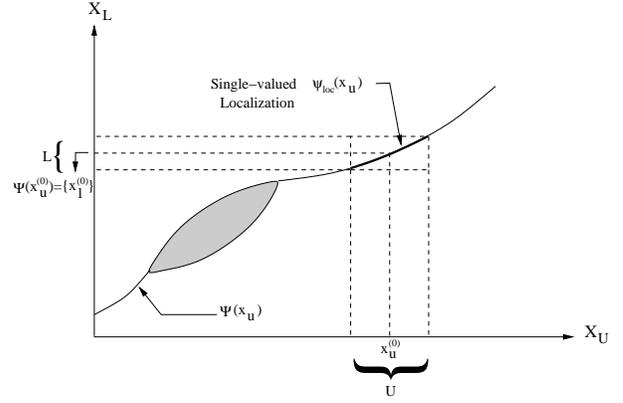,width=0.9\linewidth}
\caption{Localization around $x_{u}^{(0)}$ for $x_{l}^{(0)}$.}
\label{fig:explain0}
\end{center}
\end{figure}

As an example, Figure \ref{fig:explain0} shows the $\Psi$ mapping and its localization $\psi_{\rm loc}$ around $x_{u}^{(0)}$ for $x_l^{(0)}$. The notations discussed in the previous definition are shown in the figure to develop a graphical insight for the theory discussed above.

\subsection{Localization with Convex Constraints}\label{sec:convex-constraints}

To motivate the development, suppose that the lower level problem is of the form
\begin{equation}\label{problem1}
\minimize f_0(x_u,x_l) \text{ over all $x_l \in C$},
\end{equation}
where lower-level variables are restricted by a simple set-constraint $C$ which is assumed to be non-empty, convex, and closed in $X_L$. 

When the lower-level objective function $f_0$ is continuously differentiable, the necessary condition for $x_l$ to be a local minimum for the lower level problem is given by the standard variational inequality
\begin{equation}\label{inequality1}
\nabla_{l}f_0(x_u,x_l)+N_C(x_l)\ni 0,
\end{equation}
where 
\begin{equation}
N_C(x_l)=\{v | \langle v, x_l'-x_l\rangle\leq 0, \text{ for all } x'_l\in C\}
\end{equation}
is the normal cone to $C$ at $x_l$ and $\nabla_lf_0$ denotes the gradient of $f_0$ with respect to the lower-level decision vector. The solutions to the inequality are referred to as {\it stationary points} with respect to minimizing over $C$, regardless of whether or not they correspond to local or global minimum. Of course, in the special case when $f_0$ is also convex, the inequality yields a sufficient condition for $x_l$ to be a global minimum, but in other cases the condition is not sufficient to guarantee lower-level optimality of $x_l$. Rather the inequality is better interpreted as a tool for identifying quasi-solutions, which can be augmented with additional criteria to ensure optimality. In particular, as discussed by Dontchev and Rockafellar~\cite{dontchev09}, significant efforts have been done to develop tools that help to understand the behavior of solutions under small perturbations which we are planning to utilize while proposing our solution for solving the bilevel problem. With this purpose in mind, we introduce the following definition of a quasi-solution mapping for the lower level problem:
\begin{definition}[Quasi-solution mapping]
 The solution candidates (stationary points) to the lower level problem of form~\eqref{problem1} can be identified by set-valued mapping $\Psi^{\star}:X_U\tos X_L$,
$$
\Psi^{\star}(x_u)=\{x_l | \nabla_{l}f_0(x_u,x_l)+N_C(x_l)\ni 0\},
$$
which represents the set of stationary lower-level decisions for the given upper-level decision $x_u$. When $f_0$ is convex, $\Psi^{\star}$ coincides to the optimal solution mapping, i.e. $\Psi^{\star}=\Psi$.
\end{definition}

Whereas direct localization of $\Psi$ is difficult, it is easier to obtain a well-behaved localization for the quasi-solution mapping first, and then establish the conditions under which the obtained solutions furnish a lower-level local minimum. The approach is motivated by the fact that for variational inequalities of the above type, there are several variants of implicit function theorem that can be readily applied to obtain localizations with desired properties. Below, we present two localization-theorems for quasi-solution mappings. The first theorem shows the circumstances in which there exists a single-valued localization of $\Psi^{\star}$ such that it is Lipschitz-continuous around a given pair $(x_u^{\star},x_l^{\star})\in\gph \Psi^{\star}$. The second theorem elaborates the result further under the additional assumption that $C$ is polyhedral, which is sufficient to guarantee that for all points in the neighborhood of $x_u^{\star}$ there is a strong local minimum in the lower level problem~\eqref{problem1}. 

\begin{definition}[Lipschitz]
A single-valued localization $\psi_{\rm loc}:X_U\to X_L$ is said to be Lipschitz continuous around an upper-level decision $\bar{x}_u$ when there exists a neighborhood $U$ of $\bar{x}_u$ and a constant $\gamma\geq 0$ such that
$$
|\psi_{\rm loc}(x_u)-\psi_{\rm loc}(x_u')|\leq \gamma|x_u-x_u'| \quad \text{for all $x_u,x_u'\in U$}.
$$
\end{definition}

\begin{theorem}[Localization of quasi-solution mapping]\label{th1}
Suppose in the lower-level optimization problem~\eqref{problem1}, with $x_l^{\star}\in \Psi^{\star}(x_u^{\star})$, that 
\begin{itemize}
\item[(i)] $C$ is non-empty, convex, and closed in $X_L$, and
\item[(ii)] $f_0$ is twice continuously differentiable with respect to $x_l$, and has a strong convexity property at $(x_u^{\star},x_l^{\star})$, i.e. there exists $\gamma>0$ such that 
$$
\langle \nabla_{ll}f_0(x_u^{\star},x_l^{\star})w,w \rangle \geq \gamma |w|^2, \quad \text{for all $w\in C-C$.}
$$
\end{itemize}
Then the quasi-solution mapping $\Psi^{\star}$ has a Lipschitz continuous single-valued localization $\psi$ around $x_u^{\star}$ for $x_l^{\star}$. 
\end{theorem}
\begin{proof}
When the lower-level objective function $f_0$ is twice continuously differentiable with the inequality $\langle \nabla_{ll}f_0(x_u^{\star},x_l^{\star})w,w \rangle \geq \gamma |w|^2$ holding for all $w\in C-C$, then by Proposition 2G.4 and Exercise 2G.5 in~\cite{dontchev09} the assumptions (a) and (b) in Theorem 2G.2 are satisfied, and this gives the rest.
\end{proof}

\begin{corollary}[Localization with polyhedral constraints]\label{th2}
Suppose that in the setup of Theorem~\eqref{th1} $C$ is polyhedral. Then the additional conclusion holds that, for all $x_u$ in some neighborhood of $x_u^{\star}$, there is a strong local minimum at $x_l=\psi(x_u)$, i.e. for some $\varepsilon>0$
$$
f_0(x_u,x'_l)\geq f_0(x_u,x_l)+\frac{\varepsilon}{2}|x'_l-x_l|^2 \quad\text{for all $x'_l\in C$ near $x_l$.}
$$
\end{corollary}
\begin{proof}
See Exercise 2G.5 and Theorem 2G.3 in~\cite{dontchev09}.
\end{proof}

It is worthwhile to note that second order conditions to ensure optimality of a quasi-solution also exist for other than polyhedral sets, but the conditions are generally not available in a convenient form.

\subsection{Localization with Nonlinear Constraints}\label{sec:nonlinear-constraints}

Until now, we have considered the simple class of lower level problems where the constraint set is not allowed to depend on the upper-level decisions. In practice, however, the lower level constraints are often dictated by the upper level choices. Therefore, the above discussion needs to be extended to cover the case of general non-linear constraints that are allowed to be functions of both $x_l$ and $x_u$. Fortunately, this can be done by drawing upon the results available for constrained parametric optimization problems. 

Consider the lower level problem of the form
\begin{align*}
\minimize_{x_l\in X_L}\quad & f_0(x_u,x_l) \\
\st\quad   & f_j(x_u,x_l)\leq 0, j=1,\dots,J
\end{align*}
where the functions $f_0,f_1,\dots,f_J$ are assumed to be twice continuously differentiable. Let $L(x_u,x_l,y)=f_0(x_u,x_l)+y_1f_1(x_u,x_l)+\cdots+y_Jf_J(x_u,x_l)$ denote the Lagrangian function. Then the necessary first-order optimality condition is given by the following variational inequality
\[
f(x_u,x_l,y)+N_E(x_l,y)\ni(0,0), 
\]
where 
\[
\begin{cases}
f(x_u,x_l,y)=(\nabla_l L(x_u,x_l,y), -\nabla_y L(x_u,x_l,y)), & \\
E=X_L\times \reals_+^J 
\end{cases}
\]
The pairs $(x_l,y)$ which solve the variational inequality are called the {\it Karush-Kuhn-Tucker} pairs corresponding to the upper-level decision $x_u$. Now in the similar fashion as done in Section~\ref{sec:convex-constraints}, our interest is to establish the conditions for the existence of a mapping $\psi_{\rm loc}$ which can be used to capture the behavior of the $x_l$ component of the Karush-Kuhn-Tucker pairs as a function of the upper-level decisions $x_u$. 

\begin{theorem}[Localization with Nonlinear Constraints]\label{th:nonlinear}
For a given upper-level decision $x_u^{\star}$, let $(x^{\star}_l,y^{\star})$ denote a corresponding Karush-Kuhn-Tucker pair. Suppose that the above problem for twice continuously differentiable functions $f_i$ is such that the following regularity conditions hold
\begin{itemize}
\item[(i)] the gradients $\nabla_l f_i(x_u,x_l)$, $i\in I$ are linearly independent, and
\item[(ii)] $\langle w, \nabla_{ll}^2 L(x_u^{\star},x_l^{\star},y^{\star})w\rangle >0$ for every $w\neq 0$, $w\in M^+$,
\end{itemize}
where 
\[
\begin{cases}
 I=\{i\in[1,J] \ | \ f_i(x_u^{\star},x_l^{\star})=0\}, \\
 M^+=\{w\in \reals^n \ | \ w \perp\nabla_l f_i(x_u^{\star},x_l^{\star}) \text{ for all $i\in I$}\}.
\end{cases}
\]
Then the Karush-Kuhn-Tucker mapping $S:X_U\tos X_L\times\reals^J$,
$$
S(x_u):=\{(x_l,y) \ | \ f(x_u,x_l,y)+N_E(x_l,y)\ni(0,0)\},
$$
has a Lipschitz-continuous single-valued localization $s$ around $x_u^{\star}$ for $(x_l^{\star},y^{\star})$, $s(x_u)=(\psi_{\rm loc}(x_u),y(x_u))$, where $\psi_{\rm loc}:X_U\to X_L$ and $y: X_U\to \reals^J$. Moreover, for every $x_u$ in some neighborhood of $x_u^{\star}$ the lower-level decision component $x_l=\psi_{\rm loc}(x_u)$ gives a strong local minimum.
\end{theorem}
\begin{proof}
See e.g. Theorem 1 in~\cite{dempe97} or Theorem 2G.9 in~\cite{dontchev09}.
\end{proof}

Despite the more complicated conditions required to establish the result for general non-linear constraints case, the eventual outcome of the Theorem~\eqref{th:nonlinear} is essentially similar to Theorem~\eqref{th1} and Theorem~\eqref{th2}. That is, in the vicinity of the current optimal solution, we can express the locally optimal lower level decisions as a relatively smooth function of the upper-level decisions.

\subsection{Approximation with Localizations}

The single-valued localizations of solution mappings can be used as effective tools for alleviating computational burden in a greedy evolutionary algorithm. Instead of solving the lower level problem from scratch for every upper-level decision, we can use localizations to obtain an ``educated guess'' on the new optimal lower-level decision. The intuition for using the above results is as follows.

Suppose that $(x_u,x_l)\in\gph\Psi$ is a known lower-level optimal pair, i.e. $x_l\in\Psi(x_u)$, and the lower-level problem satisfies the regularity conditions in the previous theorems. Then for some open neighborhoods $U\subset X_U$ of $x_u$ and $L\subset X_L$ of $x_l$, there exists a uniquely determined $\gamma$-Lipschitz continuous function $\psi_{\rm loc}:U\to X_L$ such that $x'_l=\psi_{\rm loc}(x'_u)$ is the unique local optimal solution of the lower level problem in $L$ for each $x'_u\in U$. The existence of such $\psi_{\rm loc}$ leads to a direct strategy for generating new solution candidates from the existing ones. If the currently known upper level decision $x_u$ is perturbed by a small random vector $\varepsilon$ such that $x_u+\varepsilon\in U$, then the newly generated point $(x_u+\varepsilon,\psi_{\rm loc}(x_u+\varepsilon))$ gives another locally optimal solution where the change in the lower-level optimal decision vector is bounded by $|\psi_{\rm loc}(x_u+\varepsilon)-x_l|\leq \gamma|\varepsilon|$. In theory, this would allow us to reduce the bilevel optimization problem~\eqref{def:bilevel2} to that of minimizing a locally Lipschitz function $F_0(x_u,\psi_{\rm loc}(x_u))$; see e.g. \cite{dempe09} for discussion on implicit function based techniques.

However, an essential difficulty for applying the result in practice follows from the fact that the mapping $\psi_{\rm loc}$ is not known with certainty, except for a few special cases. To resolve this problem in an efficient way, we consider embedding the above results  within an evolutionary framework, where estimates of $\psi_{\rm loc}$ are produced by using samples of currently known lower level optimal points. For simplicity, suppose that we want to find an estimate of $\psi_{\rm loc}$ around current best solution $(x_u,x_l)\in\gph\Psi$, and let
$$
\mathcal{P}=\{(x_u^{(i)},x_l^{(i)})\in X_U\times X_L \ | \ x_l^{(i)}\in\Psi(x_u^{(i)}), i\in \mathcal{I}\}\subset\gph\Psi
$$ 
be a sample from the neighborhood of $(x_u,x_l)$. Then the task of finding a good estimator for $\psi_{\rm loc}$ can be viewed as an ordinary supervised learning problem:
\begin{definition} [Learning of $\psi_{\rm loc}$]
Let $\mathcal{H}$ be the hypothesis space, i.e. the set of functions that can be used to predict the optimal lower-level decision from the given upper-level decision. Given the sample $\mathcal{P}$, our goal is to choose a model $\hat{\psi}\in\mathcal{H}$ such that it minimizes the empirical error on the sample data-set, i.e.
\begin{equation}\label{eq:empirical}
\hat{\psi}=\argmin_{h \in \mathcal{H}}\sum_{i\in \mathcal{I}} L(h(x_u^{(i)}),x_l^{(i)}),
\end{equation}
where $L:X_L\times X_L\to\reals$ denotes the empirical risk function. 
\end{definition}

For a graphical illustration, see Figure \ref{fig:explain1} showing one example of a local approximation $\hat{\psi}$ around $x_{u}^{(0)}$ for $x_l$ using a quadratic function. When the exact form of the underlying mapping is unknown, the approximation generally leads to an error as one moves away from the point around which the localization is performed. In the figure, the approximation error is shown for a point $x_{u}^{(1)}$ in the neighborhood of $x_{u}^{(0)}$. This error may not be significant in the vicinity of $x_{u}^{(0)}$, and could provide a good guess for the lower-level optimal solution for a new $x_{u}$ close to $x_{u}^{(0)}$.

\begin{figure}[t]
\begin{center}
\epsfig{file=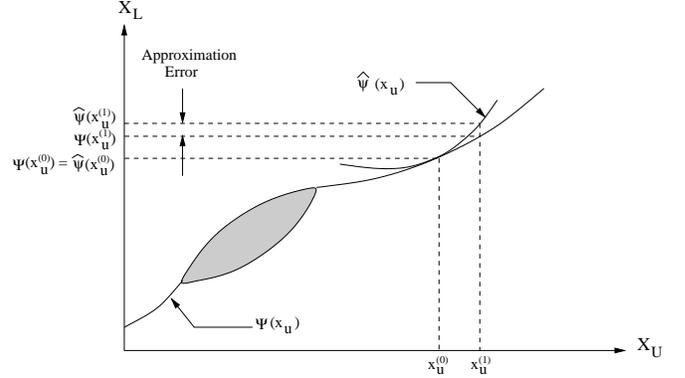,width=0.98\linewidth}
\caption{Approximation with localization around $x_{u}^{(0)}$.}
\label{fig:explain1}
\end{center}
\end{figure}
In this paper, we have chosen to use the squared prediction error as the empirical risk function when performing the approximations, i.e.
$$
L(h(x_u^{(i)}),x_l^{(i)})=|x_l^{(i)}-h(x_u^{(i)})|^2,
$$
 and at the same time we have restricted the hypothesis space $\mathcal{H}$ to consist of second-order polynomials. With these choices the empirical risk minimization problem~\eqref{eq:empirical} corresponds to an ordinary quadratic regression problem. Therefore, as long as the estimation problem is kept light enough and the evolutionary framework is such that the solution population can be used to construct a training dataset $\mathcal{P}$, the use of the estimation approach can be expected to enhance the algorithm's overall performance by reducing the number of times the lower-level optimization problem needs to be solved.

\section{Algorithm description}\label{sec:algorithm}
In this section, we provide a description for the bilevel evolutionary algorithm based on quadratic approximations (BLEAQ). The optimization strategy is based on approximation of the lower level optimal variables as a function of upper level variables. To begin with, an initial population of upper level members with random upper level variables is initialized. For each member, the lower level optimization problem is solved using a lower level optimization scheme, and optimal lower level members are noted. Based on the lower level optimal solutions achieved, a quadratic relationship between the upper level variables and each lower level optimal variable is established. If the approximation is good (in terms of mean squared error) then it can be used to predict the optimal lower level variables for any given set of upper level variables. This eliminates the requirement to solve a lower level optimization problem. However, one needs to be cautious while accepting the optimal solutions from the quadratic approximations as even a single poor solution might lead to a wrong bilevel optimum. At each generation of the algorithm, a new quadratic approximation is generated which goes on improving as the population converges towards the true optimum. At the termination of the procedure, the algorithm not only provides the optimal solutions to a bilevel problem, but also acceptably accurate functions representing the relationship between upper and lower variables at the optimum. These functions can be useful for strategy development or decision making in bilevel problems appearing in fields like game theory, economics, science and engineering. Below we provide a step-by-step procedure for the algorithm.

\begin{description}
\item[S. 1] Initialization: The algorithm starts with a random population of size $N$, which is initialized by generating the required number of upper level variables, and then executing the lower level evolutionary optimization procedure to determine the corresponding optimal lower level variables. Fitness is assigned based on upper level function value and constraints.

\item[S. 2] Tagging: Tag all the upper level members as 1 which have undergone a lower level optimization run, and others as 0.

\item[S. 3] Selection of upper level parents: Given the current population, choose the best tag 1 member as one of the parents\footnote{The choice of best tag 1 member as a parent makes the algorithm faster. However, for better exploration at upper level some other strategy may also be used.}. 
Randomly choose $2(\mu-1)$ members from the population and perform a tournament selection based on upper level function value to choose remaining $\mu-1$ parents.

\item[S. 4] Evolution at the upper level: Make the best tag 1 member as the index parent, and the members from the previous step as other parents. Then, create $\lambda$ offsprings from the chosen $\mu$ parents, using crossover and polynomial mutation operators.

\item[S. 5] Quadratic approximation: If the number of tag 1 members in the population is greater than $\frac{(dim(x_{u})+1)(dim(x_{u})+2)}{2} + dim(x_{u})$, then select all the tag 1 upper level members to construct quadratic functions\footnote{Please note that a quadratic fit in $d$ dimensions requires at least $\frac{(d+1)(d+2)}{2}$ points. However, to avoid overfitting we use at least $\frac{(d+1)(d+2)}{2}+d$ points.} to represent each of the lower level optimal variables as a function of upper level variables. If the number of tag 1 members is less than $\frac{(dim(x_{u})+1)(dim(x_{u})+2)}{2} + dim(x_{u})$ then a quadratic approximation is not performed.

\item[S. 6] Lower level optimum: If a quadratic approximation was performed in the previous step, find the lower level optimum for the offsprings using the quadratic approximation. If the mean squared error $e_{mse}$ is less than $e_{0}$(1e-3), the quadratic approximation is considered good and the offsprings are tagged as 1, otherwise they are tagged as 0. If a quadratic approximation was not performed in the previous step, execute lower level optimization runs for each of the offsprings. To execute lower level optimization for an offspring member, the closest tag 1 parent is determined. From the closest tag 1 parent, the lower level optimal member ($x_{l}^{(c)}$) is copied (Refer Section \ref{sec:closest}). Thereafter, a lower level optimization run is performed to optimize the problem using a quadratic programming approach with($x_{l}^{(c)}$) as a starting point. If the optimization is unsuccessful, use an evolutionary optimization algorithm to solve the problem. The copied lower level member, ($x_{l}^{(c)}$) is used as a population member in the lower level evolutionary optimization run. Tag the offspring as 1 for which lower level optimization is performed.

\item[S. 7] Population Update: After finding the lower level variables for the offsprings, $r$ members are chosen from the parent population. A pool of chosen $r$ members and $\lambda$ offsprings is formed. The best $r$ members from the pool replace the chosen $r$ members from the population. A termination check is performed. If the termination check is false, the algorithm moves to the next generation (Step 3).
\end{description}

\subsection{Property of two close upper level members}\label{sec:closest}
For two close upper level members, it is often expected that the lower level optimal solutions will also lie close to each other. This scenario is explained in Figure \ref{fig:closest}, where $x_{u}^{(1)}$ and $x_{u}^{(2)}$ are close to each other, and therefore their corresponding optimal lower level members are also close. On the other hand $x_{u}^{(3)}$ is far away from $x_{u}^{(1)}$ and $x_{u}^{(2)}$. Therefore, its optimal lower level member is not necessarily close to the other two lower level members. This property of the bilevel problems is utilized in the proposed algorithm. If a lower level optimization has been performed for one of the upper level members, then the corresponding lower level member is utilized while performing a lower level optimization task for another upper level member in the vicinity of the previous member.

\begin{figure}[t]
\begin{center}
\epsfig{file=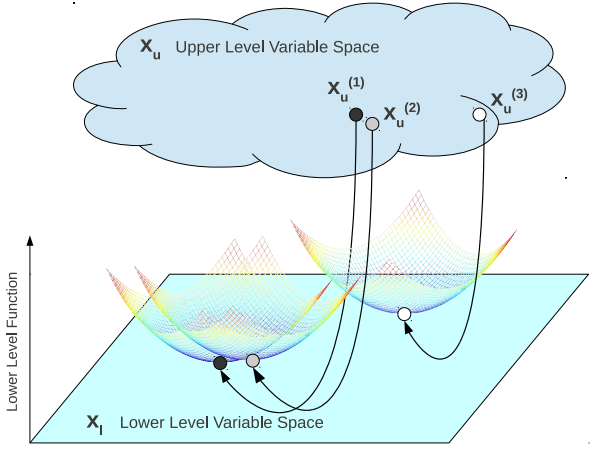,width=0.98\linewidth}
\caption{Lower level optimal solutions for three different upper level members.}
\label{fig:closest}
\end{center}
\end{figure}

\subsection{Lower Level Optimization}
At the lower level we first use a quadratic programming approach to find the optimum. If the procedure is unsuccessful we use a global optimization procedure using an evolutionary algorithm to find the optimum. The lower level optimization using evolutionary algorithm is able to handle complex optimization problems with multimodality. The fitness assignment at this level is performed based on lower level function value and constraints. The steps for the lower level optimization procedure are as follows:
\vspace{1mm}
\subsubsection{Lower level quadratic programming}
\begin{description}
\item[S. 1] Create  $\frac{(dim(x_{l})+1)(dim(x_{l})+2)}{2} + dim(x_{l})$ lower level points about $x_{l}^{(c)}$ using polynomial mutation.
\item[S. 2] Construct a quadratic approximation for lower level objective function about $x_{l}^{(c)}$ using the created points. Construct linear approximations for the lower level constraints.
\item[S. 3] Optimize the quadratic function with linear constraints using a sequentially quadratic programming approach.
\item[S. 4] Compute the value of the optimum using the quadratic approximated objective function and the true lower level objective function. If the absolute difference is less than $\delta_{min}$ and the point is feasible with respect to the constraints, accept the solution as lower level optimum, otherwise perform an evolutionary optimization search.
\end{description}

\vspace{1mm}
\subsubsection{Lower level evolutionary optimization}
\begin{description}
\item[S. 1]  If lower level evolutionary optimization is executed directly without quadratic programming, then randomly initialize $n$ lower level members. If quadratic programming is already executed but unsuccessful, then use the solution obtained using quadratic programming as one of the population members and randomly initialize other $n-1$ lower level members. The upper level variables are kept fixed for all the population members.
\item[S. 2] Choose $2\mu$ members randomly from the population, and perform a tournament selection to choose $\mu$ parents for crossover.
\item[S. 3] The best parent among $\mu$ parents is chosen as the index parent and $\lambda$ number of offsprings are produced using the crossover and mutation operators.
\item[S. 4] A population update is performed by choosing $r$ random members from the population. A pool is formed using $r$ chosen members and $\lambda$ offsprings, from which the best $r$ members are used to replace the $r$ chosen members from the population.
\item[S. 5] Next generation (Step 2) is executed if the termination criteria is not satisfied.
\end{description}

\subsection{Constraint handling}
As we know by now that a bilevel problem has two levels of optimization tasks. There may be constraints at both levels. We modify any given bilevel problem such that lower level constraints belong only to the lower level optimization task. However, at the upper level we include both upper and lower level constraints. This is done to ensure that a solution which is not feasible at the lower level cannot be feasible at the upper level, no matter whether the lower level optimization task is performed or not. While computing the overall constraint violation for a solution at the upper level, it is not taken into account whether the lower level variables are optimal or not. We use a separate tagging scheme in the algorithm to account for optimality or non-optimality of lower level variables.

The algorithm uses similar constraint handling scheme at both levels, where the overall constraint violation for any solution is the summation of the violations of all the equality and inequality constraints. A solution $x^{(i)}$ is said to `constraint dominate' \cite{debpenalty} a solution $x^{(j)}$ if any of the following conditions are true:

\begin{enumerate}
	\item Solution $x^{(i)}$ is feasible and solution $x^{(j)}$ is not.
	\item Solution $x^{(i)}$ and $x^{(j)}$ are both infeasible but solution $x^{(i)}$ has a smaller overall constraint violation.
	\item Solution $x^{(i)}$ and $x^{(j)}$ are both feasible but the objective value of $x^{(i)}$ is less than that of $x^{j}$.
\end{enumerate}


\subsection{Crossover Operator}
The crossover operator used in Step 2 is similar to the PCX operator proposed in \cite{my-cec06}. The operator creates a new solution from 3 parents as follows:
\begin{equation}
c = x^{(p)} + \omega_{\xi}d + \omega_{\eta}\frac{p^{(2)}-p^{(1)}}{2}
\label{eq:child}
\end{equation}

The terms used in the above equation are defined as follows:
\begin{itemize}
	\item $x^{(p)}$ is the {\em index\/} parent
	\item $d=x^{(p)}-g$, where $g$ is the mean of $\mu$ parents
	\item $p^{(1)}$ and $p^{(2)}$ are the other two parents
	\item $\omega_{\xi}=0.1$ and $\omega_{\eta}=\frac{dim(x^{(p)})}{||x^{(p)}-g||_{1}}$ are the two parameters.
\end{itemize}
The two parameters $\omega_{\xi}$ and $\omega_{\eta}$, describe the extent of variations along the respective directions. At the upper level, a crossover is performed only with the upper level variables and the lower level variables are determined from the quadratic function or by lower level optimization call. At the lower level, crossover is performed only with the lower level variables and the upper level variables are kept fixed as parameters.

\subsection{Termination Criteria}
The algorithm uses a variance based termination criteria at both levels. At the upper level, when the value of $\alpha_u$ described in the following equation becomes less than $\alpha_{u}^{stop}$, the algorithm terminates.
\begin{equation}
\begin{array}{l}
	\alpha_u = \sum_{i=1}^{n} \frac{\sigma^2(x_{u_{T}}^{i})}{\sigma^2(x_{u_{0}}^i)},
\end{array}
\end{equation}
where $n$ is the number of upper level variables in the bilevel optimization problem, $x_{u_T}^i : \mbox{i} \in \{1,2,\ldots,n\}$ represents the upper level variables in generation number $T$, and $x_{u_0}^i : \mbox{i} \in \{1,2,\ldots,n\}$ represents the upper level variables in the initial random population. The value of $\alpha_{u}^{stop}$ should be kept small to ensure a high accuracy.
Note that $\alpha_u$ is always greater than 0 and should be less than 1 most of the times.

A similar termination scheme is used for the lower level evolutionary algorithm. The value for $\alpha_l$ is given by:
\begin{equation}
\begin{array}{l}
	\alpha_l = \sum_{i=1}^{m} \frac{\sigma^2(x_{l_{t}}^{i})}{\sigma^2(x_{l_{0}}^i)},
\end{array}
\end{equation}
where $m$ is the number of variables at the lower level, $x_{l_t}^i : \mbox{i} \in \{1,2,\ldots,m\}$ represents the lower level variables in generation number $t$, and $x_{l_0}^i : \mbox{i} \in \{1,2,\ldots,m\}$ represents the lower level variables in the initial random population for a particular lower level optimization run.
A high accuracy is desired particularly at the lower level, therefore the value for $\alpha_l$ should be kept low. Inaccurate lower level solutions may mislead the algorithm in case of a conflict between the two levels.

\subsection{Parameters and Platform}
The parameters in the algorithm are fixed as $\mu=3$, $\lambda=2$ and $r=2$ at both levels. Crossover probability is fixed at $0.9$ and the mutation probability is $0.1$. The upper level population size $N$ and the lower level population size $n$ are fixed at 50 for all the problems. The values for $\alpha_{u}^{stop}$ and $\alpha_{l}^{stop}$ are fixed as $1e-5$ at both upper and the lower levels.

The code for the algorithm is written in MATLAB, and all the computations have been performed on a machine with 64 bit UNIX kernel, 2.6GHz quad-core Intel Core i7 processor and 8GB of 1600MHz DDR3 RAM.

\subsection{Example}
To demonstrate the working of the algorithm we use a simple bilevel optimization problem introduced in Example \ref{example1}. The steps of the BLEAQ procedure on this example are illustrated in Figure \ref{fig:exampleProblem}. The figure shows the steps S1 to S7 with the help of four part figures PF1, PF2, PF3 and PF4. The algorithm starts with a random population of five upper level members that are shown with white circles in PF1. The figure also shows the actual $\Psi$-mapping that is unknown at the start of the algorithm. For each member $x_u$, lower level optimization is executed to get the corresponding optimal $x_l$. This provides the $(x_u,x_l)$ pair shown with gray circles. The members are tagged 1 if the lower level optimization is successful. This is followed by PF2 showing selection and evolution of the upper level members leading to $\lambda=2$ offsprings from $\mu=3$ parents. The offsprings are shown with white squares. Thereafter, PF3 shows the quadratic approximation of the $\Psi$-mapping. At this stage, the lower level optimal solutions corresponding to the offsprings are no longer generated from lower level optimization rather the approximate $\Psi$-mapping is used to obtain approximate lower level solutions. The $(x_u,x_l)$ pairs for the offsprings are shown using gray squares. Since the offspring pairs are generated from a poor approximation, we tag the offspring members as 0. Next, in PF4 a population update is performed based on upper level fitness that leads to replacement of poor fitness members with better members. 
From this example, we observe that an approximate mapping helps in focusing the search in the better regions at the upper level by avoiding frequent lower level optimization runs. Tagging is important so that the algorithm does not loose all the tag 1 members. At any generation, the elite member is always the best tag 1 member, because the tag 0 members are known to be non-optimal at the lower level.
\begin{figure*}[t]
\begin{center}
\epsfig{file=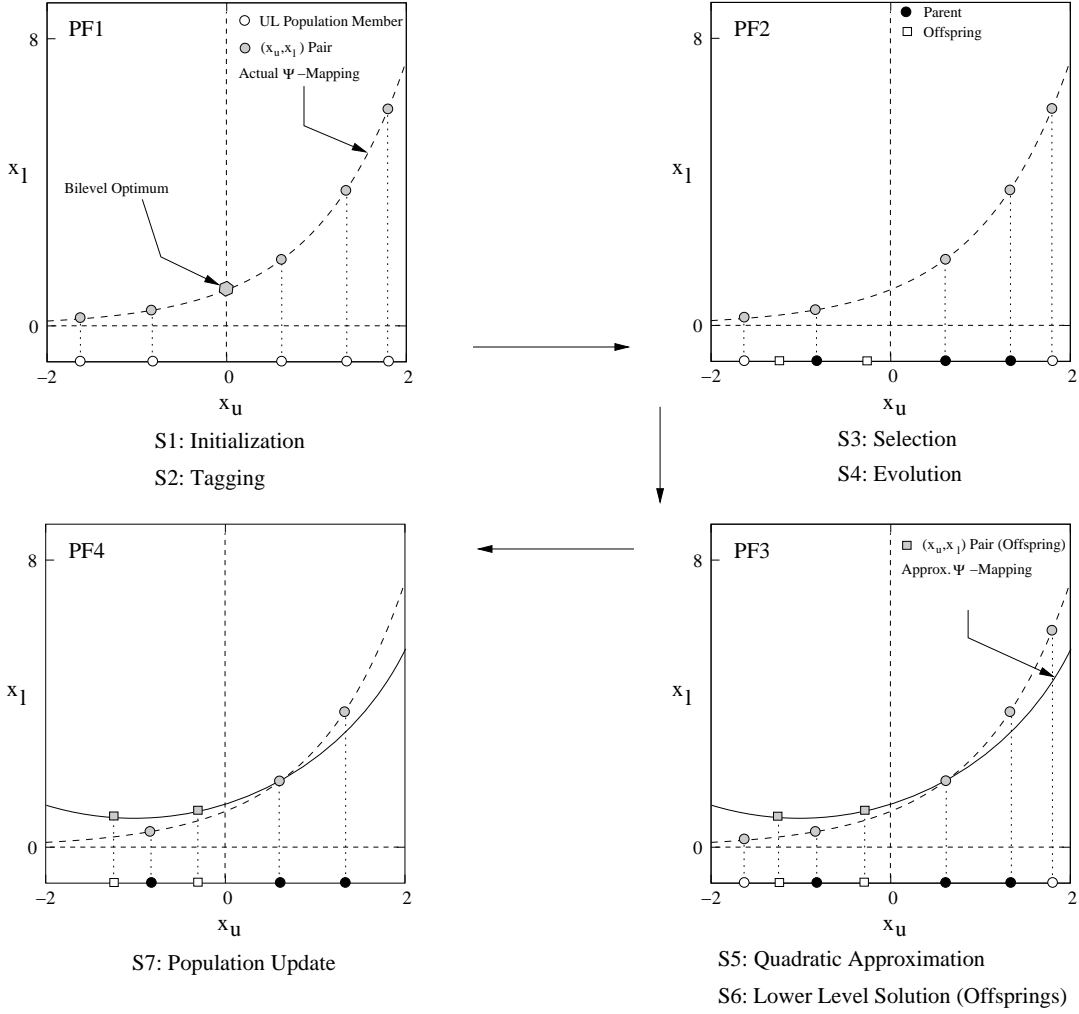,width=0.8\linewidth}
\caption{Graphical illustration of the steps followed by BLEAQ on Example \ref{example1}.}
\label{fig:exampleProblem}
\end{center}
\end{figure*}


\section{SMD problems}
A set of test problems (SMD) for single objective bilevel optimization has recently been proposed in \cite{my-cec12a}. The proposed problems are scalable in terms of the number of decision variables. In this section, we discuss the construction of these test problems in brief and provide a description of the proposed SMD test problems. The construction of the test problem involves splitting the upper and lower level functions into three components. Each of the components is specialized for inducing a certain kind of difficulty into the bilevel test problem. The functions are chosen based on the difficulties required in terms of convergence at the two levels, and interaction between the two levels. A generic bilevel test problem is stated as follows:

\begin{equation}
\begin{array}{l}
F_0(x_u,x_l) = F_0^1(x_{u1}) + F_0^2(x_{l1}) + F_0^3(x_{u2},x_{l2}) \\
f_0(x_u,x_l) = f_0^1(x_{u1}, x_{u2}) + f_0^2(x_{l1}) + f_0^3(x_{u2},x_{l2})\\

\mbox{where}\\
     \quad \quad x_u = (x_{u1}, x_{u2}) \quad \mbox{and} \quad x_l = (x_{l1}, x_{l2})
\end{array}
\end{equation}

\begin{table*}
\caption{Overview of test-problem framework components}\label{tab:framework}
\begin{minipage}{1.0\linewidth}
\begin{footnotesize}
\begin{center}
\begin{tabular}{|c|c|c|c|}
\multicolumn{4}{c}{Panel A: Decomposition of decision variables}\\
\hline
\multicolumn{2}{|c|}{Upper-level variables} & \multicolumn{2}{|c|}{Lower-level variables} \\
\hline
Vector & Purpose & Vector & Purpose \\
\hline\hline
$x_{u1}$ & Complexity on upper-level & $x_{l1}$ & Complexity on lower-level \\
$x_{u2}$ & Interaction with lower-level & $x_{l2}$ & Interaction with upper-level \\
\hline

\multicolumn{4}{c}{}\\

\multicolumn{4}{c}{Panel B: Decomposition of objective functions}\\
\hline
\multicolumn{2}{|c|}{Upper-level objective function} & \multicolumn{2}{|c|}{Lower-level objective function} \\
\hline
Component & Purpose & Component & Purpose \\
\hline\hline
 $F_0^1(x_{u1})$& Difficulty in convergence & $f_0^1(x_{u1}, x_{u2})$ & Functional dependence \\
 $F_0^2(x_{l1})$& Conflict / co-operation & $f_0^2(x_{l1})$ & Difficulty in convergence\\
 $F_0^3(x_{u2},x_{l2})$& Difficulty in interaction & $f_0^3(x_{u2},x_{l2})$ & Difficulty in interaction \\
\hline
\end{tabular}
\end{center}
\end{footnotesize}
\end{minipage}
\end{table*}

The above equations contain three terms at both levels. Table~\ref{tab:framework} provides a summary for the role played by each term in the equations. The upper level and lower level variables are broken into two smaller vectors (refer Panel A in Table~\ref{tab:framework}), such that, vectors $x_{u1}$ and $x_{l1}$ are used to induce complexities at the upper and lower levels independently, and vectors $x_{u2}$ and $x_{l2}$ are responsible to induce complexities because of interaction. The upper and lower level functions are decomposed such that each of the components is specialized for a certain purpose only (refer Panel B in Table~\ref{tab:framework}). The term $F_1(x_{u1})$, at the upper level, is responsible for inducing difficulty in convergence solely at the upper level, and the term $f_2(x_{l1})$, at the lower level, is responsible for inducing difficulty in convergence solely at the lower level. The term $F_2(x_{l1})$ decides if there is a conflict or a cooperation between the two levels. The terms $F_3(x_{l2}, x_{u2})$ and $f_3(x_{l2}, x_{u2})$ induce difficulties because of interaction at the two levels, though $F_3(x_{l2}, x_{u2})$ may also be used to induce a cooperation or a conflict. Finally, $f_1(x_{u1}, x_{u1})$ is a fixed term at the lower level which does not induce any difficulties, rather helps to create a functional dependence between lower level optimal solutions and the upper level variables.

\subsection{SMD1}
SMD1 is a test problem with cooperation between the two levels. The lower level optimization problem is a simple convex optimization task. The upper level is convex with respect to upper level variables and optimal lower level variables.
\begin{equation}
\begin{array}{l}
F_0^1 = \sum_{i=1}^{p} (x_{u1}^{i})^2\\
F_0^2 = \sum_{i=1}^{q} (x_{l1}^{i})^2\\
F_0^3 = \sum_{i=1}^{r} (x_{u2}^{i})^2 + \sum_{i=1}^{r} (x_{u2}^{i} - \tan x_{l2}^{i})^2\\
f_0^1 = \sum_{i=1}^{p} (x_{u1}^{i})^2\\
f_0^2 = \sum_{i=1}^{q} (x_{l1}^{i})^2\\
f_0^3 = \sum_{i=1}^{r} (x_{u2}^{i} - \tan x_{l2}^{i})^2\\
\end{array}
\end{equation}
The range of variables is as follows,
\begin{equation}
\begin{array}{l}
x_{u1}^{i} \in [-5,10], \hspace{2mm} \forall \hspace{2mm} i \in \{1,2,\ldots,p\}\\
x_{u2}^{i} \in [-5,10], \hspace{2mm} \forall \hspace{2mm} i \in \{1,2,\ldots,r\}\\
x_{l1}^{i} \in [-5,10], \hspace{2mm} \forall \hspace{2mm} i \in \{1,2,\ldots,q\}\\
x_{l2}^{i} \in (\frac{-\pi}{2},\frac{\pi}{2}), \hspace{2mm} \forall \hspace{2mm} i \in \{1,2,\ldots,r\}
\end{array}
\end{equation}
Relationship between upper level variables and lower level optimal variables is given as follows,
\begin{equation}
\begin{array}{l}
x_{l1}^{i} = 0, \hspace{2mm} \forall \hspace{2mm} i \in \{1,2,\ldots,p\}\\
x_{l2}^{i} = \tan^{-1} x_{u2}^{i}, \hspace{2mm} \forall \hspace{2mm} i \in \{1,2,\ldots,r\}
\end{array}
\end{equation}
The values of the variables at the optimum are $x_u=0$ and $x_l$ is obtained by the relationship given above. Both upper and lower level functions are equal to $0$ at the optimum. Figure \ref{fig:smd1} shows the contours of the upper level function with respect to the upper and lower level variables for a four variable test problem with $dim(x_{u1})=dim(x_{u2})=dim(x_{l1})=dim(x_{l2})=1$. Figure \ref{fig:smd1} shows the contours of the upper level function at each $x_u$ in Sub-figure P assuming that the lower level variables are optimal. Sub-figure S shows the behavior of the upper level function with respect to $x_l$ at optimal $x_u$.

\begin{figure}
\begin{center}
\epsfig{file=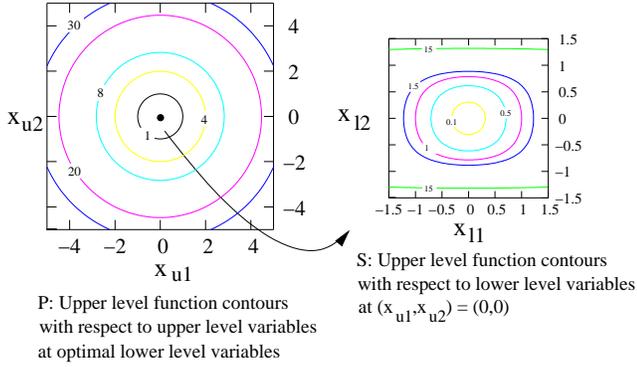,width=0.98\linewidth}
\caption{Upper level function contours for a four variable SMD1 test problem.}
\label{fig:smd1}
\end{center}
\end{figure}

\subsection{SMD2}
SMD2 is a test problem with a conflict between the upper level and lower level optimization tasks. The lower level optimization problem is a convex optimization task. An inaccurate lower level optimum may lead to upper level function value better than the true optimum for the bilevel problem. The upper level is convex with respect to upper level variables and optimal lower level variables.
\begin{equation}
\begin{array}{l}
F_0^1 = \sum_{i=1}^{p} (x_{u1}^{i})^2\\
F_0^2 = - \sum_{i=1}^{q} (x_{l1}^{i})^2\\
F_0^3 = \sum_{i=1}^{r} (x_{u2}^{i})^2 - \sum_{i=1}^{r} (x_{u2}^{i} - \log x_{l2}^{i})^2\\
f_0^1 = \sum_{i=1}^{p} (x_{u1}^{i})^2\\
f_0^2 = \sum_{i=1}^{q} (x_{l1}^{i})^2\\
f_0^3 = \sum_{i=1}^{r} (x_{u2}^{i} - \log x_{l2}^{i})^2\\
\end{array}
\end{equation}
The range of variables is as follows,
\begin{equation}
\begin{array}{l}
x_{u1}^{i} \in [-5,10], \hspace{2mm} \forall \hspace{2mm} i \in \{1,2,\ldots,p\}\\
x_{u2}^{i} \in [-5,1], \hspace{2mm} \forall \hspace{2mm} i \in \{1,2,\ldots,r\}\\
x_{l1}^{i} \in [-5,10], \hspace{2mm} \forall \hspace{2mm} i \in \{1,2,\ldots,q\}\\
x_{l2}^{i} \in (0,e], \hspace{2mm} \forall \hspace{2mm} i \in \{1,2,\ldots,r\}
\end{array}
\end{equation}
Relationship between upper level variables and lower level optimal variables is given as follows,
\begin{equation}
\begin{array}{l}
x_{l1}^{i} = 0, \hspace{2mm} \forall \hspace{2mm} i \in \{1,2,\ldots,q\}\\
x_{l2}^{i} = \log^{-1} x_{u2}^{i}, \hspace{2mm} \forall \hspace{2mm} i \in \{1,2,\ldots,r\}
\end{array}
\end{equation}
The values of the variables at the optimum are $x_u=0$ and $x_l$ is obtained by the relationship given above. Both upper and lower level functions are equal to $0$ at the optimum. Figure \ref{fig:smd2} represents the same information as in Figure \ref{fig:smd1} for a four variable bilevel test problem.

\begin{figure}
\begin{center}
\epsfig{file=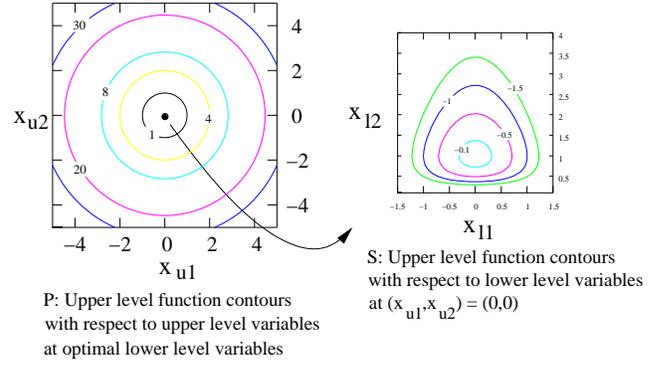,width=0.98\linewidth}
\caption{Upper level function contours for a four variable SMD2 test problem.}
\label{fig:smd2}
\end{center}
\end{figure}

\subsection{SMD3}
SMD3 is a test problem with a cooperation between the two levels. The difficulty introduced is in terms of multi-modality at the lower level which contains the Rastrigin's function. The upper level is convex with respect to upper level variables and optimal lower level variables.
\begin{equation}
\begin{array}{l}
F_0^1 = \sum_{i=1}^{p} (x_{u1}^{i})^2\\
F_0^2 = \sum_{i=1}^{q} (x_{l1}^{i})^2\\
F_0^3 = \sum_{i=1}^{r} (x_{u2}^{i})^2 + \sum_{i=1}^{r} ((x_{u2}^{i})^2 - \tan x_{l2}^{i})^2\\
f_0^1 = \sum_{i=1}^{p} (x_{u1}^{i})^2\\
f_0^2 = q + \sum_{i=1}^{q} \left(\left(x_{l1}^{i}\right)^2 - \cos 2\pi x_{l1}^{i}\right)\\
f_0^3 = \sum_{i=1}^{r} ((x_{u2}^{i})^2 - \tan x_{l2}^{i})^2\\
\end{array}
\end{equation}
The range of variables is as follows,
\begin{equation}
\begin{array}{l}
x_{u1}^{i} \in [-5,10], \hspace{2mm} \forall \hspace{2mm} i \in \{1,2,\ldots,p\}\\
x_{u2}^{i} \in [-5,10], \hspace{2mm} \forall \hspace{2mm} i \in \{1,2,\ldots,r\}\\
x_{l1}^{i} \in [-5,10], \hspace{2mm} \forall \hspace{2mm} i \in \{1,2,\ldots,q\}\\
x_{l2}^{i} \in (\frac{-\pi}{2},\frac{\pi}{2}), \hspace{2mm} \forall \hspace{2mm} i \in \{1,2,\ldots,r\}
\end{array}
\end{equation}
Relationship between upper level variables and lower level optimal variables is given as follows,
\begin{equation}
\begin{array}{l}
x_{l1}^{i} = 0, \hspace{2mm} \forall \hspace{2mm} i \in \{1,2,\ldots,q\}\\
x_{l2}^{i} = \tan^{-1} (x_{u2}^{i})^2, \hspace{2mm} \forall \hspace{2mm} i \in \{1,2,\ldots,r\}
\end{array}
\end{equation}
The values of the variables at the optimum are $x_u=0$ and $x_l$ is obtained by the relationship given above. Both upper and lower level functions are equal to $0$ at the optimum. Figure \ref{fig:smd3} shows the contours of the upper level function at each $x_u$ in Sub-figure P assuming that the lower level variables are optimal. Sub-figure S shows the behavior of the lower level function at optimal $x_u$.

\begin{figure}
\begin{center}
\epsfig{file=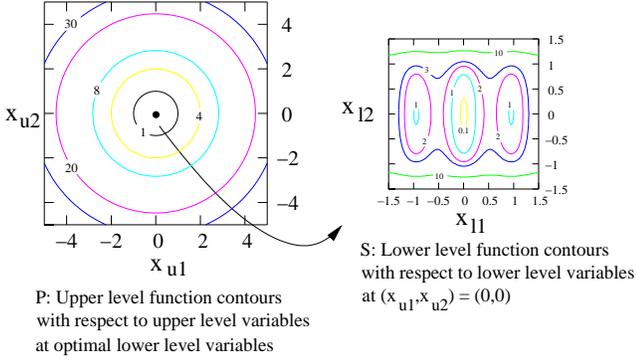,width=0.98\linewidth}
\caption{Upper and lower level function contours for a four variable SMD3 test problem.}
\label{fig:smd3}
\end{center}
\end{figure}

\subsection{SMD4}
SMD4 is a test problem with a conflict between the two levels. The difficulty is in terms of multi-modality at the lower level which contains the Rastrigin's function. The upper level is convex with respect to upper level variables and optimal lower level variables.
\begin{equation}
\begin{array}{l}
F_0^1 = \sum_{i=1}^{p} (x_{u1}^{i})^2\\
F_0^2 = - \sum_{i=1}^{q} (x_{l1}^{i})^2\\
F_0^3 = \sum_{i=1}^{r} (x_{u2}^{i})^2 - \sum_{i=1}^{r} (|x_{u2}^{i}| - \log (1+x_{l2}^{i}))^2\\
f_0^1 = \sum_{i=1}^{p} (x_{u1}^{i})^2\\
f_0^2 = q + \sum_{i=1}^{q} \left(\left(x_{l1}^{i}\right)^2 - \cos 2\pi x_{l1}^{i}\right)\\
f_0^3 = \sum_{i=1}^{r} (|x_{u2}^{i}| - \log(1+x_{l2}^{i}))^2\\
\end{array}
\end{equation}
The range of variables is as follows,
\begin{equation}
\begin{array}{l}
x_{u1}^{i} \in [-5,10], \hspace{2mm} \forall \hspace{2mm} i \in \{1,2,\ldots,p\}\\
x_{u2}^{i} \in [-1,1], \hspace{2mm} \forall \hspace{2mm} i \in \{1,2,\ldots,r\}\\
x_{l1}^{i} \in [-5,10], \hspace{2mm} \forall \hspace{2mm} i \in \{1,2,\ldots,q\}\\
x_{l2}^{i} \in [0,e], \hspace{2mm} \forall \hspace{2mm} i \in \{1,2,\ldots,r\}
\end{array}
\end{equation}
Relationship between upper level variables and lower level optimal variables is given as follows,
\begin{equation}
\begin{array}{l}
x_{l1}^{i} = 0, \hspace{2mm} \forall \hspace{2mm} i \in \{1,2,\ldots,q\}\\
x_{l2}^{i} = \log^{-1} |x_{u2}^{i}| - 1, \hspace{2mm} \forall \hspace{2mm} i \in \{1,2,\ldots,r\}
\end{array}
\end{equation}
The values of the variables at the optimum are $x_u=0$ and $x_l$ is obtained by the relationship given above. Both upper and lower level functions are equal to $0$ at the optimum. Figure~\ref{fig:smd4} represents the same information as in Figure~\ref{fig:smd3} for a four variable bilevel problem.

\begin{figure}
\begin{center}
\epsfig{file=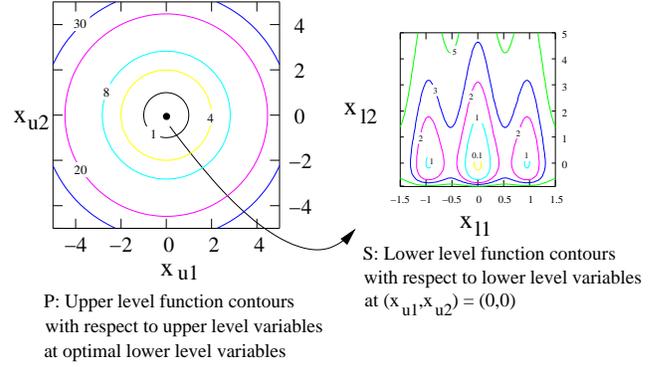,width=0.98\linewidth}
\caption{Upper and lower level function contours for a four variable SMD4 test problem.}
\label{fig:smd4}
\end{center}
\end{figure}

\subsection{SMD5}
SMD5 is a test problem with a conflict between the two levels. The difficulty introduced is in terms of multi-modality and convergence at the lower level. The lower level problem contains the banana function such that the global optimum lies in a long, narrow, flat parabolic valley. The upper level is convex with respect to upper level variables and optimal lower level variables.
\begin{equation}
\begin{array}{l}
F_0^1 = \sum_{i=1}^{p} (x_{u1}^{i})^2\\
F_0^2 = - \sum_{i=1}^{q} \left( \left(x_{l1}^{i+1} - \left(x_{l1}^{i}\right)^2\right) + \left(x_{l1}^{i} - 1\right)^2 \right)\\
F_0^3 = \sum_{i=1}^{r} (x_{u2}^{i})^2 - \sum_{i=1}^{r} (|x_{u2}^{i}| - (x_{l2}^{i})^2)^2\\
f_0^1 = \sum_{i=1}^{p} (x_{u1}^{i})^2\\
f_0^2 = \sum_{i=1}^{q} \left( \left(x_{l1}^{i+1} - \left(x_{l1}^{i}\right)^2\right) + \left(x_{l1}^{i} - 1\right)^2 \right)\\
f_0^3 = \sum_{i=1}^{r} (|x_{u2}^{i}| - (x_{l2}^{i})^2)^2\\
\end{array}
\end{equation}
The range of variables is as follows,
\begin{equation}
\begin{array}{l}
x_{u1}^{i} \in [-5,10], \hspace{2mm} \forall \hspace{2mm} i \in \{1,2,\ldots,p\}\\
x_{u2}^{i} \in [-5,10], \hspace{2mm} \forall \hspace{2mm} i \in \{1,2,\ldots,r\}\\
x_{l1}^{i} \in [-5,10], \hspace{2mm} \forall \hspace{2mm} i \in \{1,2,\ldots,q\}\\
x_{l2}^{i} \in [-5,10], \hspace{2mm} \forall \hspace{2mm} i \in \{1,2,\ldots,r\}
\end{array}
\end{equation}
Relationship between upper level variables and lower level optimal variables is given as follows,
\begin{equation}
\begin{array}{l}
x_{l1}^{i} = 1, \hspace{2mm} \forall \hspace{2mm} i \in \{1,2,\ldots,q\}\\
x_{l2}^{i} = \sqrt{|x_{u2}^{i}|}, \hspace{2mm} \forall \hspace{2mm} i \in \{1,2,\ldots,r\}
\end{array}
\end{equation}
The values of the variables at the optimum are $x_u=0$ and $x_l$ is obtained by the relationship given above. Both upper and
lower level functions are equal to $0$ at the optimum.

\subsection{SMD6}
SMD6 is a test problem with a conflict between the two levels. The problem contains infinitely many global solutions at the lower level, for any given upper level vector. Out of the entire global solution set, there is only a single lower level point which corresponds to the best upper level function value.
\begin{equation}
\begin{array}{l}
F_0^1 = \sum_{i=1}^{p} (x_{u1}^{i})^2 \\
F_0^2 = - \sum_{i=1}^{q} (x_{l1}^{i})^2 + \sum_{i=q+1}^{q+s} (x_{l1}^{i})^2\\
F_0^3 = \sum_{i=1}^{r} (x_{u2}^{i})^2 - \sum_{i=1}^{r} (x_{u2}^{i} - x_{l2}^{i})^2\\
f_0^1 = \sum_{i=1}^{p} (x_{u1}^{i})^2\\
f_0^2 = \sum_{i=1}^{q} (x_{l1}^{i})^2 + \sum_{i=q+1, i=i+2}^{q+s-1} (x_{l1}^{i+1} - x_{l1}^{i})^2\\
f_0^3 = \sum_{i=1}^{r} (x_{u2}^{i} - x_{l2}^{i})^2\\
\end{array}
\end{equation}
The range of variables is as follows,
\begin{equation}
\begin{array}{l}
x_{u1}^{i} \in [-5,10], \hspace{2mm} \forall \hspace{2mm} i \in \{1,2,\ldots,p\}\\
x_{u2}^{i} \in [-5,10], \hspace{2mm} \forall \hspace{2mm} i \in \{1,2,\ldots,r\}\\
x_{l1}^{i} \in [-5,10], \hspace{2mm} \forall \hspace{2mm} i \in \{1,2,\ldots,q+s\}\\
x_{l2}^{i} \in [-5,10], \hspace{2mm} \forall \hspace{2mm} i \in \{1,2,\ldots,r\}
\end{array}
\end{equation}
Relationship between upper level variables and lower level optimal variables is given as follows,
\begin{equation}
\begin{array}{l}
x_{l1}^{i} = 0, \hspace{2mm} \forall \hspace{2mm} i \in \{1,2,\ldots,q\}\\
x_{l2}^{i} = x_{u2}^{i}, \hspace{2mm} \forall \hspace{2mm} i \in \{1,2,\ldots,r\}
\end{array}
\end{equation}
The values of the variables at the optimum are $x_u=0$ and $x_l$ is obtained by the relationship given above. Both upper and
lower level functions are equal to $0$ at the optimum.


\section{Standard test problems}
Tables \ref{tab:testset1} and \ref{tab:testset2} define a set of 10 standard test problems collected from the literature. The dimensions of the upper and lower level variables are given in the first column, and the problem formulation is defined in the second column. The third column provides the best known solution available in the literature for the chosen test problems.

\begin{table*}
\caption{Description of the selected standard test problems (TP1-TP5).}
{\small
\begin{tabular}{p{.15\textwidth}  >{$\displaystyle}p{.55\textwidth}<{$}  >{$\displaystyle}p{.15\textwidth}<{$}}
\toprule
$\mbox{Problem}$ & \mbox{Formulation} & \mbox{Best Known Sol.} \\
\toprule
TP1 \\ $n=2$ $m=2$ & 
\begin{array}{l}
\underset{{(x_u,x_l)}}{\mbox{Minimize}} \hspace{1mm} F_0(x_u,x_l) = (x_{u}^{1}-30)^2 + (x_{u}^{2}-20)^2 - 20 x_{l}^{1} + 20 x_{l}^{2}, \\
\mbox{s.t.}\\
\quad x_l \in \underset{(x_l)}{\argmin}\left\lbrace
	\begin{array}{l} 
 	f_0(x_u,x_l)=(x_{u}^{1}-x_{l}^{1})^2+(x_{u}^{2}-x_{l}^{2})^2\\
 		0 \le x_{l}^{i} \le 10, \quad i=1,2
	\end{array}
	 \right\rbrace, \\
\quad x_{u}^{1}+2x_{u}^{2} \ge 30, x_{u}^{1}+x_{u}^{2} \le 25, x_{u}^{2} \le 15\\
\end{array} 
   & \begin{array}{l} F_0 = 225.0 \\ f_0 = 100.0
\end{array}
\\ \midrule
TP2 \\ $n=2$ $m=2$ & 
\begin{array}{l}
\underset{{(x_u,x_l)}}{\mbox{Minimize}} \hspace{1mm} F_0(x_u,x_l) = 2 x_{u}^{1} + 2 x_{u}^{2} - 3 x_{l}^{1} - 3 x_{l}^{2} - 60, \\
\mbox{s.t.}\\
\quad x_l \in \underset{(x_l)}{\argmin}\left\lbrace 
	\begin{array}{l}
	 f_0(x_u,x_l)=(x_{l}^{1} - x_{u}^{1} + 20)^2+(x_{l}^{2} - x_{u}^{2} + 20)^2 \\
	x_{u}^{1} - 2 x_{l}^{1} \ge 10, x_{u}^{2} - 2 x_{l}^{2} \ge 10\\
	-10 \ge x_{l}^{i} \ge 20, \quad i=1,2
	\end{array}
	 \right\rbrace, \\
\quad x_{u}^{1}+x_{u}^{2}+x_{l}^{1}-2 x_{l}^{2} \le 40,\\
\quad 0 \le x_{u}^{i} \le 50, \quad i=1,2.
\end{array} 
   & \begin{array}{l} F_0 = 0.0 \\ f_0 = 100.0
\end{array}
\\ \midrule
TP3 \\ $n=2$ $m=2$ & 
\begin{array}{l}
\underset{{(x_u,x_l)}}{\mbox{Minimize}} \hspace{1mm} F_0(x_u,x_l) = -(x_{u}^{1})^2 - 3 (x_{u}^{2})^2 - 4 x_{l}^{1} + (x_{l}^{2})^2, \\
\mbox{s.t.}\\
\quad x_l \in \underset{(x_l)}{\argmin}\left\lbrace 
	\begin{array}{l}
	 f_0(x_u,x_l)=2 (x_{u}^{1})^2 + (x_{l}^{1})^2 - 5 x_{l}^{2} \\
	(x_{u}^{1})^2 - 2 x_{u}^{1} + (x_{u}^{2})^2 - 2 x_{l}^{1} + x_{l}^{2} \ge -3\\
	x_{u}^{2} + 3 x_{l}^{1} - 4 x_{l}^{2} \ge 4\\
	0 \le x_{l}^{i}, \quad i=1,2
	\end{array}
	 \right\rbrace, \\
\quad (x_{u}^{1})^2 + 2 x_{u}^{2} \le 4,\\
\quad 0 \le x_{u}^{i}, \quad i=1,2
\end{array} 
   & \begin{array}{l} F_0 = -18.6787 \\ f_0 = -1.0156
\end{array}
\\ \midrule
TP4 \\ $n=2$ $m=3$ & 
\begin{array}{l}
\underset{{(x_u,x_l)}}{\mbox{Minimize}} \hspace{1mm} F_0(x_u,x_l) = -8 x_{u}^{1} - 4 x_{u}^{2} + 4 x_{l}^{1} - 40 x_{l}^{2} - 4 x_{l}^{3}, \\
\mbox{s.t.}\\
\quad x_l \in \underset{(x_l)}{\argmin}\left\lbrace 
	\begin{array}{l}
	 f_0(x_u,x_l)=x_{u}^{1} + 2 x_{u}^{2} + x_{l}^{1} + x_{l}^{2} + 2 x_{l}^{3}\\
	x_{l}^{2} + x_{l}^{3} - x_{l}^{1} \le 1\\
	2 x_{u}^{1} - x_{l}^{1} + 2 x_{l}^{2} - 0.5 x_{l}^{3} \le 1\\
	2 x_{u}^{2} + 2 x_{l}^{1} - x_{l}^{2} - 0.5 x_{l}^{3} \le 1\\
	0 \le x_{l}^{i}, \quad i=1,2,3
	\end{array}
	 \right\rbrace, \\
\quad 0 \le x_{u}^{i}, \quad i=1,2
\end{array} 
   & \begin{array}{l} F_0 = -29.2 \\ f_0 = 3.2
\end{array}
\\ \bottomrule
TP5 \\ $n=2$ $m=2$ & 
\begin{array}{l}
\underset{{(x_u,x_l)}}{\mbox{Minimize}} \hspace{1mm} F_0(x_u,x_l) = r t(x_u) x_u - 3 x_{l}^{1} - 4 x_{l}^{2} + 0.5 t(x_l) x_l, \\
\mbox{s.t.}\\
\quad x_l \in \underset{(x_l)}{\argmin}\left\lbrace 
	\begin{array}{l}
	 f_0(x_u,x_l)=0.5 t(x_l) h x_l - t(b(x_u)) x_l\\
	-0.333 x_{l}^{1} + x_{l}^{2} - 2 \le 0\\
	x_{l}^{1} - 0.333 x_{l}^{2} - 2 \le 0\\
	0 \le x_{l}^{i}, \quad i=1,2
	\end{array}
	 \right\rbrace, \\
\mbox{where}\\
\quad h = \left( \begin{array}{cc} 1 & 3\\ 3 & 10\\ \end{array} \right), 
b(x) = \left( \begin{array}{cc} -1 & 2\\ 3 & -3\\ \end{array} \right)x, 
r = 0.1\\
\quad t(\cdot) \mbox{ denotes transpose of a vector}
\end{array} 
   & \begin{array}{l} F_0 = -3.6 \\ f_0 = -2.0
\end{array}
\\ \midrule
\end{tabular}
}
\vspace{-1mm}
\label{tab:testset1}
\end{table*}

\begin{table*}
\caption{Description of the selected standard test problems (TP6-TP10).}
{\small 
\begin{tabular}{p{.12\textwidth} >{$\displaystyle}p{.65\textwidth}<{$} >{$\displaystyle}p{.15\textwidth}<{$}}
\toprule
$\mbox{Problem}$ & \mbox{Formulation} & \mbox{Best Known Sol.} \\
\toprule
TP6 \\ $n=1$ $m=2$ & 
\begin{array}{l}
\underset{{(x_u,x_l)}}{\mbox{Minimize}} \hspace{1mm} F_0(x_u,x_l) = (x_{u}^{1} - 1)^2 + 2 x_{l}^{1} - 2 x_{u}^{1}, \\
\mbox{s.t.}\\
\quad x_l \in \underset{(x_l)}{\argmin}\left\lbrace 
	\begin{array}{l}
	 f_0(x_u,x_l)=(2 x_{l}^{1}-4)^2 +\\ (2 x_{l}^{2} - 1)^2 + x_{u}^{1} x_{l}^{1}\\
	4 x_{u}^{1} + 5 x_{l}^{1} + 4 x_{l}^{2} \le 12\\
	4 x_{l}^{2} - 4 x_{u}^{1} - 5 x_{l}^{1} \le -4\\
	4 x_{u}^{1} - 4 x_{l}^{1} + 5 x_{l}^{2} \le 4\\
	4 x_{l}^{1} - 4 x_{u}^{1} + 5 x_{l}^{2} \le 4\\
	0 \le x_{l}^{i}, \quad i=1,2
	\end{array}
	 \right\rbrace, \\
\quad 0 \le x_{u}^{1}
\end{array} 
   & \begin{array}{l} F_0 = -1.2091 \\ f_0 = 7.6145
\end{array}
\\ \midrule
TP7 \\ $n=2$ $m=2$ & 
\begin{array}{l}
\underset{{(x_u,x_l)}}{\mbox{Minimize}} \hspace{1mm} F_0(x_u,x_l) = -\frac{(x_{u}^{1}+x_{l}^{1})(x_{u}^{2}+x_{l}^{2})}{1 + x_{u}^{1} x_{l}^{1} + x_{u}^{2} x_{l}^{2}}, \\
\mbox{s.t.}\\
\quad x_l \in \underset{(x_l)}{\argmin}\left\lbrace 
	\begin{array}{l}
	 f_0(x_u,x_l)=\frac{(x_{u}^{1}+x_{l}^{1})(x_{u}^{2}+x_{l}^{2})}{1 + x_{u}^{1} x_{l}^{1} + x_{u}^{2} x_{l}^{2}}\\
	0 \le x_{l}^{i} \le x_{u}^{i}, \quad i=1,2
	\end{array}
	 \right\rbrace, \\
\quad (x_{u}^{1})^2 + (x_{u}^{2})^2 \le 100\\
\quad x_{u}^{1} - x_{u}^{2} \le 0\\
\quad 0 \le x_{u}^{i}, \quad i=1,2
\end{array} 
   & \begin{array}{l} F_0 = -1.96 \\ f_0 = 1.96
\end{array}
\\ \midrule
TP8  \\ $n=2$ $m=2$ & 
\begin{array}{l}
\underset{{(x_u,x_l)}}{\mbox{Minimize}} \hspace{1mm} F_0(x_u,x_l) = |2 x_{u}^{1} + 2 x_{u}^{2} - 3 x_{l}^{1} - 3 x_{l}^{2} - 60|, \\
\mbox{s.t.}\\
\quad x_l \in \underset{(x_l)}{\argmin}\left\lbrace 
	\begin{array}{l}
	 f_0(x_u,x_l)=(x_{l}^{1} - x_{u}^{1} + 20)^2 +\\ (x_{l}^{2} - x_{u}^{2} + 20)^2\\
	2x_{l}^{1} - x_{u}^{1} + 10 \le 0\\
	2x_{l}^{2} - x_{u}^{2} + 10 \le 0\\
	-10 \le x_{l}^{i} \le 20, \quad i=1,2
	\end{array}
	 \right\rbrace, \\
\quad x_{u}^{1} + x_{u}^{2} + x_{l}^{1} - 2 x_{l}^{2} \le 40\\
\quad 0 \le x_{u}^{i} \le 50, \quad i=1,2
\end{array} 
   & \begin{array}{l} F_0 = 0.0 \\ f_0 = 100.0
\end{array}
\\ \midrule
TP9 \\ $n=10$ $m=10$ & 
\begin{array}{l}
\underset{{(x_u,x_l)}}{\mbox{Minimize}} \hspace{1mm} F_0(x_u,x_l) = \sum_{i=1}^{10} \left( |x_{u}^{i}-1| + |x_{l}^{i}|\right), \\
\mbox{s.t.}\\
\quad x_l \in \underset{(x_l)}{\argmin}\left\lbrace 
	\begin{array}{l}
	 f_0(x_u,x_l)=e^{\left(1+\frac{1}{4000}\sum_{i=1}^{10} (x_{l}^{i})^2 - \prod_{i=1}^{10} \cos(\frac{x_{l}^{i}}{\sqrt{i}})\right) \sum_{i=1}^{10} (x_{u}^{i})^{2}}\\
\quad -\pi \le x_{l}^{i} \le \pi, \quad i=1, 2 \ldots, 10
	\end{array}
	 \right\rbrace,
\end{array} 
   & \begin{array}{l} F_0 = 0.0 \\ f_0 = 1.0
\end{array}
\\ \midrule
TP10  \\ $n=10$ $m=10$ & 
\begin{array}{l}
\underset{{(x_u,x_l)}}{\mbox{Minimize}} \hspace{1mm} F_0(x_u,x_l) = \sum_{i=1}^{10} \left( |x_{u}^{i}-1| + |x_{l}^{i}|\right), \\
\mbox{s.t.}\\
\quad x_l \in \underset{(x_l)}{\argmin}\left\lbrace 
	\begin{array}{l}
	 f_0(x_u,x_l)=e^{\left(1+\frac{1}{4000}\sum_{i=1}^{10} (x_{l}^{i} x_{u}^{i})^2 - \prod_{i=1}^{10} \cos(\frac{x_{l}^{i} x_{u}^{i}}{\sqrt{i}})\right)}\\
\quad -\pi \le x_{l}^{i} \le \pi, \quad i=1, 2 \ldots, 10
	\end{array}
	 \right\rbrace,
\end{array} 
   & \begin{array}{l} F_0 = 0.0 \\ f_0 = 1.0
\end{array}
\\ \bottomrule
\end{tabular}
}
\vspace{-1mm}
\label{tab:testset2}
\end{table*}

\section{Results}
In this section, we provide the results obtained on the SMD test problems and standard test problems using BLEAQ and other approaches. The section is divided into two parts. The first part contains the results obtained on the SMD test problems using BLEAQ and the nested procedure. It was difficult to identify an approach from the existing literature, which can efficiently handle the SMD test problems with 10 variables. Therefore, we have chosen the nested approach as the benchmark for this study. The second part contains the results obtained on 10 standard constrained test problems, which have been studied by others in the past. We compare our method against the approaches proposed by Wang et al. (2005,2011) \cite{wang05,wang11}. Both approaches are able to handle all the standard test problems successfully.

\subsection{Results for SMD test problems}
We report the results obtained by the nested procedure and the BLEAQ approach in this sub-section on 10-dimensional unconstrained SMD test problems. For test problems SMD1 to SMD5 we choose $p=3$, $q=3$ and $r=2$, and for SMD6 we choose $p=3$, $q=1$, $r=2$ and $s=2$. Tables~\ref{tab:sub1table1} and ~\ref{tab:sub1table2} provide the results obtained using the nested approach \cite{my-cec12a} on these test problems. The nested approach uses an evolutionary algorithm at both levels, where for every upper level member a lower level optimization problem is solved. It successfully handles all the test problems. However, as can be observed from Table~\ref{tab:sub1table1}, the number of function evaluations required are very high at the lower level. We use the nested approach as a benchmark to assess the savings obtained from the BLEAQ approach.

Tables~\ref{tab:sub1table3} and~\ref{tab:sub1table4} provide the results obtained using the BLEAQ approach. The 4\textsuperscript{th} and 5\textsuperscript{th} columns in Table~\ref{tab:sub1table3} give the median function evaluations required at the upper and lower levels respectively. The numbers in the brackets represent a ratio of the function evaluations required using nested approach against the function evaluations required using BLEAQ. The nested approach requires 2 to 5 times more function evaluations at the upper level, and more than 10 times function evaluations at the lower level as compared to BLEAQ. Similar results are provided in terms of mean in the 6\textsuperscript{th} and 7\textsuperscript{th} columns of the same table. Both approaches are able to successfully handle all the test problems. However, BLEAQ dominates the nested approach by a significant margin, particularly in terms of the function evaluations at the lower level.

\begin{table*}[hbt]
\caption{Function evaluations (FE) for the upper level (UL) and the lower level (LL) from 31
  runs of nested approach.} 
\label{tab:sub1table1}
{\small\begin{center}
\begin{tabular}{|c|c|c|c|c|c|c|c|c|} \hline
Pr. No.	&	\multicolumn{2}{|c|}{Best}	&	\multicolumn{2}{|c|}{Median}	&	\multicolumn{2}{|c|}{Mean} & \multicolumn{2}{|c|}{Worst}	\\	\cline{2-9}
	&		\multicolumn{1}{|c|}{Total LL}	&	\multicolumn{1}{|c|}{Total UL}	&	\multicolumn{1}{|c|}{Total LL}	&	\multicolumn{1}{|c|}{Total UL}	&	\multicolumn{1}{|c|}{Total LL}	&\multicolumn{1}{|c|}{Total UL} & \multicolumn{1}{|c|}{Total LL}	&\multicolumn{1}{|c|}{Total UL}	\\	
	&	\multicolumn{1}{|c|}{FE} 	&
        \multicolumn{1}{|c|}{FE}	&\multicolumn{1}{|c|}{FE}         &\multicolumn{1}{|c|}{FE}        &
        \multicolumn{1}{|c|}{FE}& \multicolumn{1}{|c|}{FE} &  \multicolumn{1}{|c|}{FE}& \multicolumn{1}{|c|}{FE}	\\ \hline	
SMD1	&	807538	&	1180	&	1693710	&	2497	&	1782864.94	&	2670.87	&	2436525	&	3356	\\	\hline
SMD2	&	940746	&	1369	&	1524671	&	2309	&	1535369.41	&	2500.74	&	2548509	&	3972	\\	\hline
SMD3	&	862708	&	1113	&	1443053	&	2101	&	1405744.86	&	2218.22	&	1883302	&	3143	\\	\hline
SMD4	&	528564	&	727	&	1051430	&	1614	&	985106.94	&	1627.05	&	1567362	&	2009	\\	\hline
SMD5	&	1216411	&	1540	&	1825140	&	2992	&	1937586.63	&	3100.91	&	3107135	&	4177	\\	\hline
SMD6	&	1209859	&	1618	&	2398020	&	2993	&	2497097.41	&	3012.01	&	3202710	&	4861	\\	\hline
\end{tabular}
\end{center}}
\end{table*}

\begin{table*}[hbt]
\caption{Accuracy for the upper and lower levels, and the lower level calls from 31
  runs of nested approach.} 
\label{tab:sub1table2}
\begin{center}
\begin{tabular}{|c|c|c|c|c|c|c|c|c|} \hline
Pr. No. & Median & Median & Median & Mean & Mean & Mean & &\\ \cline{2-7}
	& UL Accuracy & LL Accuracy & LL Calls & UL Accuracy & LL Accuracy & LL Calls &  $\frac{\mbox{Med LL Evals}}{\mbox{Med LL Calls}}$ & $\frac{\mbox{Mean LL Evals}}{\mbox{Mean LL Calls}}$ \\ \hline
SMD1	&	0.005365	&	0.001616	&	2497	&	0.005893	&	0.001467	&	2670.87	&	678.33	&	667.52	\\	\hline
SMD2	&	0.001471	&	0.000501	&	2309	&	0.001582	&	0.000539	&	2500.74	&	660.24	&	613.97	\\	\hline
SMD3	&	0.008485	&	0.002454	&	2101	&	0.009660	&	0.002258	&	2218.22	&	686.87	&	633.73	\\	\hline
SMD4	&	0.008140	&	0.002866	&	1614	&	0.008047	&	0.002530	&	1627.05	&	651.54	&	605.46	\\	\hline
SMD5	&	0.001285	&	0.003146	&	2992	&	0.001311	&	0.002904	&	3100.91	&	610.07	&	624.84	\\	\hline
SMD6	&	0.009403	&	0.007082	&	2993	&	0.009424	&	0.008189	&	3012.01	&	801.17	&	829.05	\\	\hline
\end{tabular}
\end{center}
\end{table*}

\begin{table*}[hbt]
\caption{Function evaluations (FE) for the upper level (UL) and the lower level (LL) from 31
  runs with BLEAQ.} 
\label{tab:sub1table3}
{\small\begin{center}
\begin{tabular}{|c|c|c|c|c|c|c|c|c|} \hline
Pr. No.	&	\multicolumn{2}{|c|}{Best Func. Evals.}	&	\multicolumn{2}{|c|}{Median Func. Evals.} & \multicolumn{2}{|c|}{Mean Func. Evals.}	&	\multicolumn{2}{|c|}{Worst Func. Evals.}	\\	\cline{2-9}
	&		\multicolumn{1}{|c|}{LL}	&	\multicolumn{1}{|c|}{UL} & \multicolumn{1}{|c|}{LL}	&	\multicolumn{1}{|c|}{UL}	&	\multicolumn{1}{|c|}{LL}	&	\multicolumn{1}{|c|}{UL}	&	\multicolumn{1}{|c|}{LL}	&\multicolumn{1}{|c|}{UL}	\\	
	&	\multicolumn{1}{|c|}{} 	&
        \multicolumn{1}{|c|}{}	&\multicolumn{1}{|c|}{(Savings)}         &\multicolumn{1}{|c|}{(Savings)}        &
        \multicolumn{1}{|c|}{(Savings)}& \multicolumn{1}{|c|}{(Savings)} & \multicolumn{1}{|c|}{}& \multicolumn{1}{|c|}{}	\\ \hline	
SDM1	&	89966	&	589	&	110366	(15.35)	&	780	(3.20)	&	117282.14	(15.20)	&	782.55	(3.41)	&	192835	&	1636	\\	\hline
SDM2	&	67589	&	364	&	92548	(16.47)	&	615	(3.76)	&	98392.73	(15.60)	&	629.05	(3.98)	&	141868	&	1521	\\	\hline
SDM3	&	107516	&	590	&	128493	(11.23)	&	937	(2.24)	&	146189.13	(9.62)	&	906.99	(2.45)	&	145910	&	1132	\\	\hline
SDM4	&	58604	&	391	&	74274	(14.16)	&	735	(2.20)	&	73973.84	(13.32)	&	743.58	(2.19)	&	101832	&	1139	\\	\hline
SDM5	&	96993	&	311	&	127961	(14.26)	&	633	(4.73)	&	132795.14	(14.59)	&	626.44	(4.95)	&	206718	&	1316	\\	\hline
SDM6	&	90574	&	640	&	125833	(19.06)	&	970	(3.09)	&	123941.14	(20.15)	&	892.92	(3.37)	&	196966	&	1340	\\	\hline
\end{tabular}
\end{center}}
\end{table*}

\begin{table*}[hbt]
\caption{Accuracy for the upper and lower levels, and the lower level calls from 31
  runs with BLEAQ.} 
\label{tab:sub1table4}
\begin{center}
\begin{tabular}{|c|c|c|c|c|c|c|c|c|} \hline
Pr. No. & Median & Median & Median & Mean & Mean & Mean & &\\ \cline{2-7}
	& UL Accuracy & LL Accuracy & LL Calls & UL Accuracy & LL Accuracy & LL Calls &  $\frac{\mbox{Med LL Evals}}{\mbox{Med LL Calls}}$ & $\frac{\mbox{Mean LL Evals}}{\mbox{Mean LL Calls}}$ \\ \hline
SDM1	&	0.006664	&	0.003347	&	507	&	0.006754	&	0.003392	&	467.84	&	217.71	&	250.69	\\	\hline
SDM2	&	0.003283	&	0.002971	&	503	&	0.003416	&	0.002953	&	504.15	&	184.05	&	195.16	\\	\hline
SDM3	&	0.009165	&	0.004432	&	601	&	0.008469	&	0.003904	&	585.60	&	213.71	&	249.64	\\	\hline
SDM4	&	0.007345	&	0.002796	&	538	&	0.006817	&	0.002499	&	543.27	&	138.07	&	136.16	\\	\hline
SDM5	&	0.004033	&	0.003608	&	527	&	0.004257	&	0.004074	&	525.08	&	242.76	&	252.90	\\	\hline
SDM6	&	0.000012	&	0.000008	&	505	&	0.000012	&	0.000008	&	546.34	&	249.17	&	226.86	\\	\hline
\end{tabular}
\end{center}
\end{table*}

\subsection{Convergence Study}
In this sub-section, we examine the convergence of the algorithm towards the bilevel optimal solution. As test cases we use the first two SMD test problems with 10-dimensions. The results are presented in Figures \ref{fig:smd1-convergence} and \ref{fig:smd2-convergence} from a sample run. The figures show the progress of the elite member at each of the generations. The upper plot shows the convergence towards the upper level optimal function value, and the lower plot shows the convergence towards the lower level optimal function value. The algorithm preserves the elite member at the upper level, so we observe a continuous improvement in the upper plot. However, in a bilevel problem an improvement in the elite at the upper level does not guarantee a continuous improvement in the lower level function value. Therefore, we observe that the lower level convergence plot is not a continuously reducing plot rather contains small humps.
\begin{figure*}
\begin{minipage}[t]{0.47\linewidth}
\begin{center}
\epsfig{file=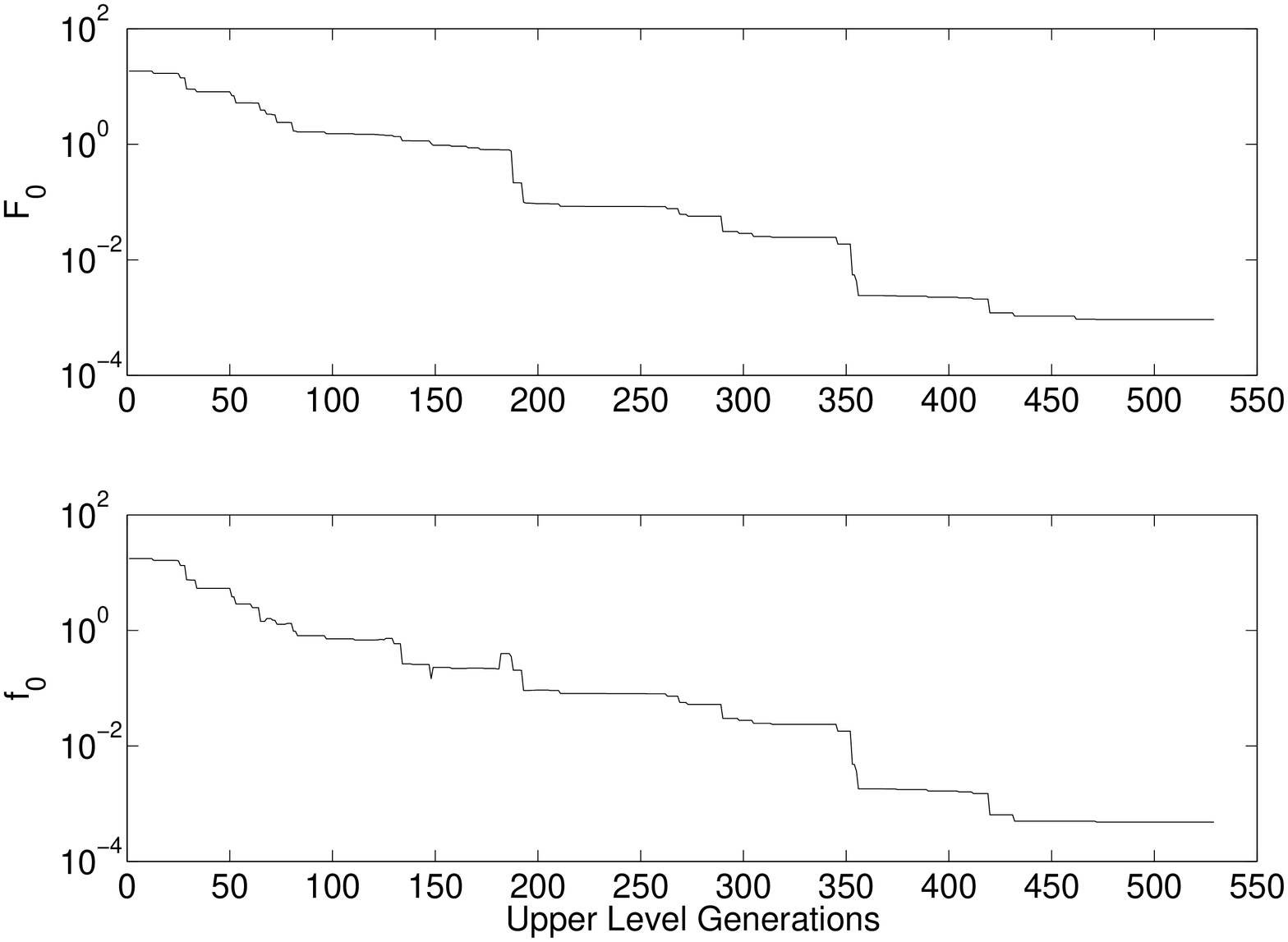,width=\linewidth}
\end{center}
\caption{Convergence plot for the upper and lower level function values for the elite member obtained from a sample run of BLEAQ on SMD1 problem.}
\label{fig:smd1-convergence}
\end{minipage}\hfill
\begin{minipage}[t]{0.47\linewidth}
\begin{center}
\epsfig{file=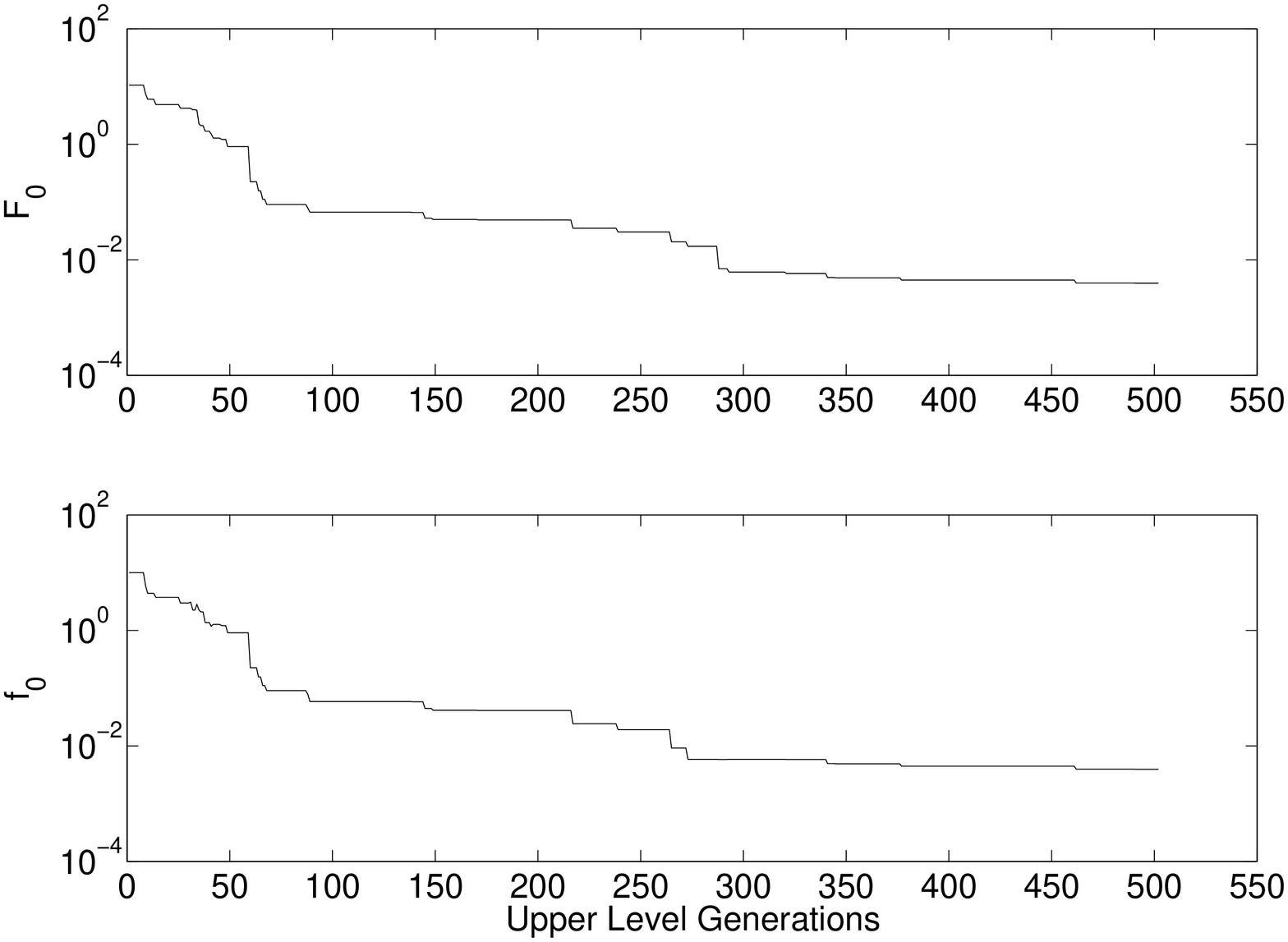,width=\linewidth} 
\end{center}
\caption{Convergence plot for the upper and lower level function values for the elite member obtained from a sample run of BLEAQ on SMD2 problem.}
\label{fig:smd2-convergence}
\end{minipage}
\end{figure*}
Figures \ref{fig:smd1-convergence-multiple} and \ref{fig:smd2-convergence-multiple} provide the mean convergence for SMD1 and SMD2 from 31 runs of the algorithm. The upper curve represents the mean of the upper level function values for the elite members, and the lower curve represents the mean of the lower level function values for the elite members obtained from multiple runs of the algorithm. The bars show the minimum and maximum function values for the elite members at each generation. We observe that for all the runs the algorithm converges to the optimum function values at both levels.
\begin{figure*}
\begin{minipage}[t]{0.47\linewidth}
\begin{center}
\epsfig{file=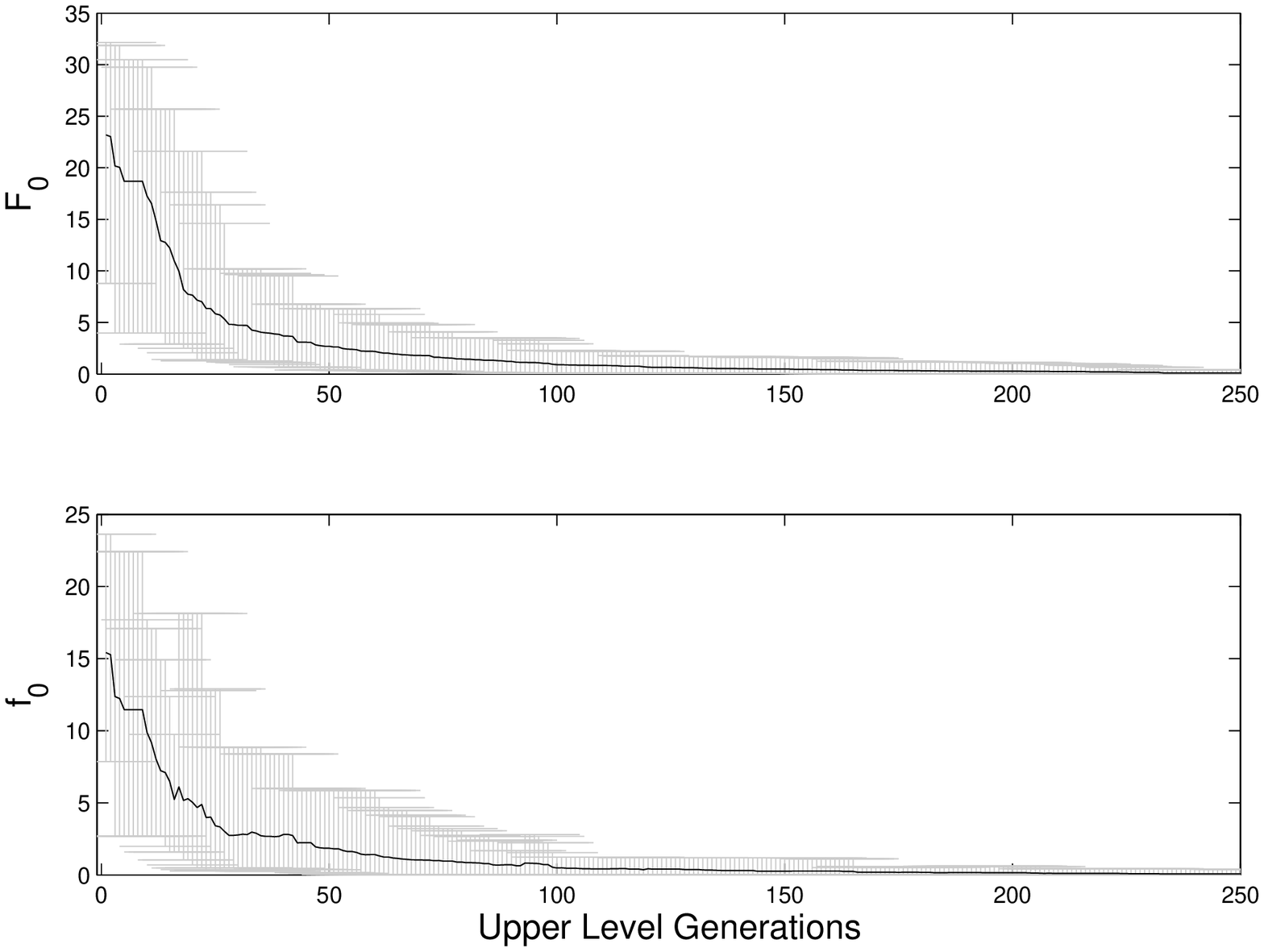,width=\linewidth}
\end{center}
\caption{Convergence plot for mean upper level function value and mean lower level function value for elite members obtained from 31 runs of BLEAQ on SMD1 problem.}
\label{fig:smd1-convergence-multiple}
\end{minipage}\hfill
\begin{minipage}[t]{0.47\linewidth}
\begin{center}
\epsfig{file=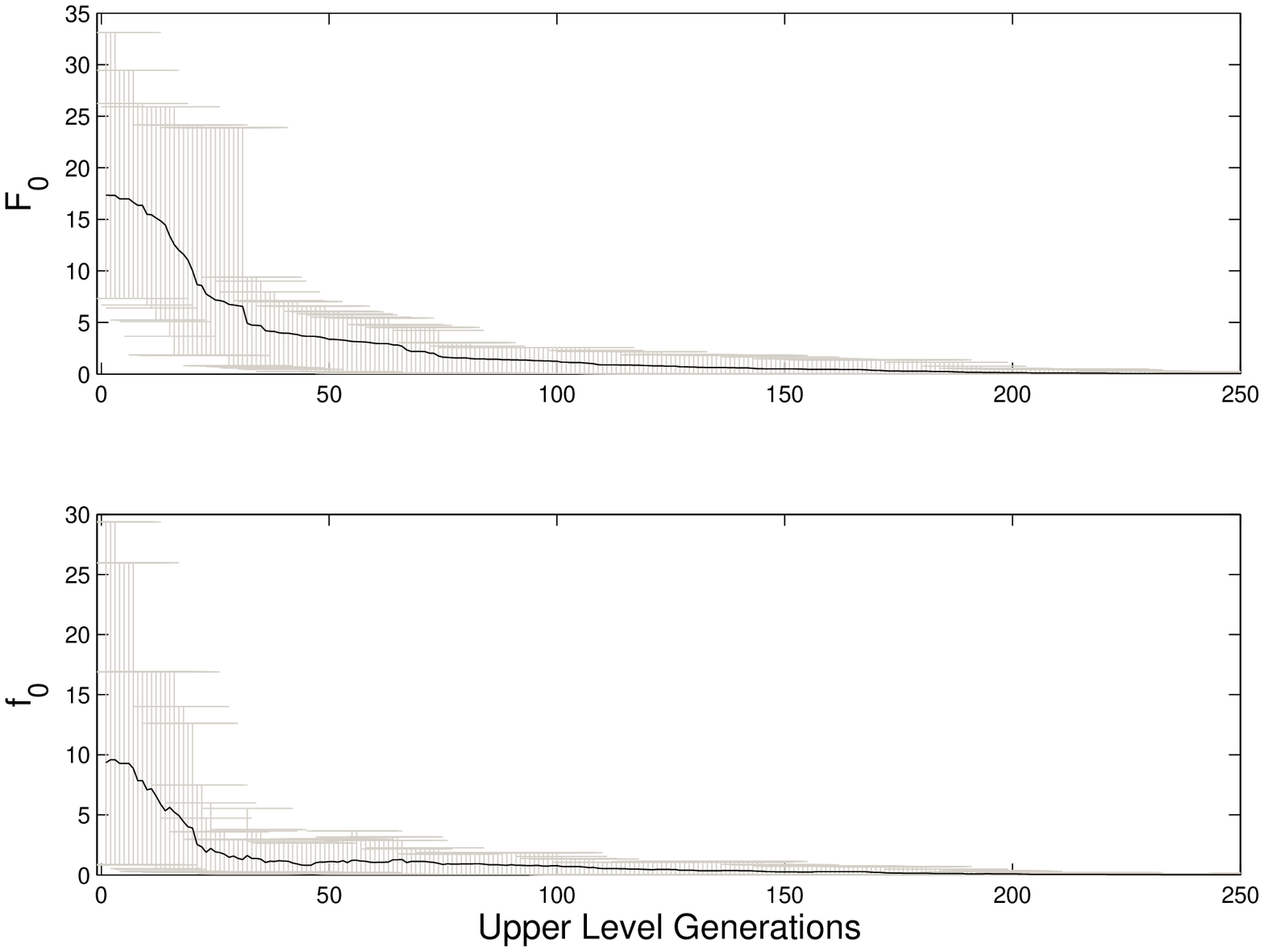,width=\linewidth} 
\end{center}
\caption{Convergence plot for mean upper level function value and mean lower level function value for elite members obtained from 31 runs of BLEAQ on SMD2 problem.}
\label{fig:smd2-convergence-multiple}
\end{minipage}
\end{figure*}

Next, we evaluate the algorithm's performance in approximating the actual $\Psi$ mapping. For this we choose the problem SMD1, which has the following $\Psi$ mapping.
\begin{equation}
\begin{array}{l}
x_{l1}^{i} = 0, \hspace{2mm} \forall \hspace{2mm} i \in \{1,2,\ldots,p\}\\
x_{l2}^{i} = \tan^{-1} x_{u2}^{i}, \hspace{2mm} \forall \hspace{2mm} i \in \{1,2,\ldots,r\}
\end{array}
\end{equation}
We set $p=3$, $q=3$ and $r=2$ to get a 10-dimensional SMD1 test problem. The $\Psi$ mapping for SMD1 test problem is variable separable. Therefore, in order to show the convergence of the approximation on a 2-d plot, we choose the variable $x_{l2}^{1}$, and show its approximation with respect to $x_{u2}^{1}$. The true relationship between the variables is shown in Figure \ref{fig:quadApprox}, along with the quadratic approximations generated at various generations. It can be observed that the approximations in the vicinity of the true bilevel optimum improve with increasing number of generations.

\begin{figure*}[t]
\begin{center}
\epsfig{file=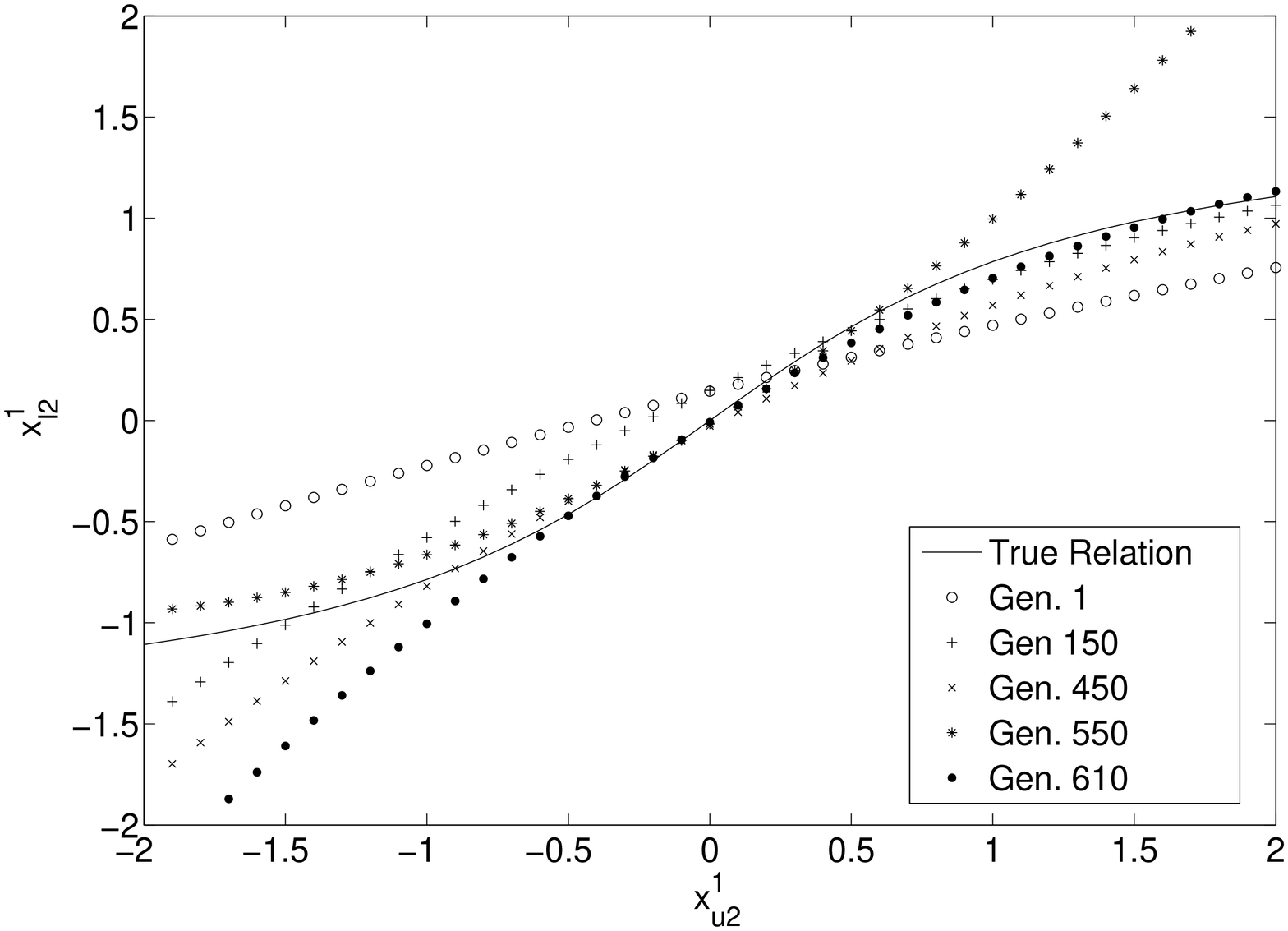,width=0.65\linewidth}
\caption{Quadratic relationship convergence.}
\label{fig:quadApprox}
\end{center}
\end{figure*}

\subsection{Results for constrained test problems}
In this sub-section, we report the results for 10 standard constrained test problems chosen from the literature. We compare our results against the approaches proposed in \cite{wang05,wang11}. The reason for the choice of the two approaches as benchmarks is that both approaches were successful in solving all the chosen constrained test problems. The results obtained using BLEAQ, WJL \cite{wang05} approach, WLD \cite{wang11} approach and nested approach have been furnished in Tables~\ref{tab:sub2table1}, ~\ref{tab:sub2table2} and~\ref{tab:sub2table3}. 

Table~\ref{tab:sub2table1} shows the minimum, median, mean and maximum function evaluations required to solve the chosen problems using the BLEAQ approach. Table~\ref{tab:sub2table2} provides the accuracy obtained at both levels, in terms of absolute difference from the best known solution to a particular test problem.
The table also reports the median and mean of the lower level calls, and the median and mean of lower level function evaluations required per lower level call from 31 runs of the algorithm on each test problem. Finally, Table~\ref{tab:sub2table3} compares the mean function evaluations at the upper and lower levels required by BLEAQ against that required by WJL, WLD and nested approach. In both papers \cite{wang05,wang11}, authors have reported the function evaluations as sum of function evaluations for the two levels. The reported performance metric in the papers is the average function evaluations from multiple runs of their algorithm. We have computed a similar metric for BLEAQ, and the results are reported in the table. It can be observed that both
WJL and WLD require close to an order of magnitude times more function evaluations for most of the test problems. This clearly demonstrates the efficiency gain obtained using the BLEAQ approach. It also suggests that the mathematical insights used along with the evolutionary principles in the BLEAQ approach are helpful in converging quickly towards the bilevel optimal solution.

\begin{table*}[!hbt]
\caption{Function evaluations (FE) for the upper level (UL) and the lower level (LL) from 31
  runs with BLEAQ.} 
\vspace{-1mm}
\label{tab:sub2table1}
{\small\begin{center}
\begin{tabular}{|c|c|c|c|c|c|c|c|c|} \hline
Pr. No.	&	\multicolumn{2}{|c|}{Best Func. Evals.}	&	\multicolumn{2}{|c|}{Median Func. Evals.}	&	\multicolumn{2}{|c|}{Mean Func. Evals.}	&	\multicolumn{2}{|c|}{Worst Func. Evals.}	\\	\cline{2-9}
	&		\multicolumn{1}{|c|}{LL}	&	\multicolumn{1}{|c|}{UL}	&	\multicolumn{1}{|c|}{LL}	&	\multicolumn{1}{|c|}{UL} & \multicolumn{1}{|c|}{LL}	&	\multicolumn{1}{|c|}{UL}	&	\multicolumn{1}{|c|}{LL}	&\multicolumn{1}{|c|}{UL}	\\	\hline
TP1	&	14115	&	718	&	15041	&	780	&	14115.37	&	695.09	&	24658	&	1348	\\	\hline
TP2	&	12524	&	1430	&	14520	&	1434	&	13456.42	&	1314.57	&	16298	&	2586	\\	\hline
TP3	&	4240	&	330	&	4480	&	362	&	3983.47	&	392.24	&	6720	&	518	\\	\hline
TP4	&	14580	&	234	&	15300	&	276	&	15006.41	&	278.72	&	15480	&	344	\\	\hline
TP5	&	10150	&	482	&	15700	&	1302	&	14097.59	&	1305.41	&	15936	&	1564	\\	\hline
TP6	&	14667	&	230	&	17529	&	284	&	16961.32	&	256.74	&	21875	&	356	\\	\hline
TP7	&	234622	&	3224	&	267784	&	4040	&	268812.56	&	4158.62	&	296011	&	5042	\\	\hline
TP8	&	10796	&	1288	&	12300	&	1446	&	10435.71	&	1629.77	&	18086	&	2080	\\	\hline
TP9	&	85656	&	496	&	96618	&	660	&	92843.85	&	672.84	&	107926	&	746	\\	\hline
TP10	&	87722	&	530	&	101610	&	618	&	99754.59	&	602.65	&	114729	&	692	\\	\hline
\end{tabular}
\vspace{-1mm}
\end{center}}
\end{table*}

\begin{table*}[!hbt]
\caption{Accuracy for the upper and lower levels, and the lower level calls from 31
  runs with BLEAQ.} 
\vspace{1mm}
\label{tab:sub2table2}
\begin{center}
\begin{tabular}{|c|c|c|c|c|c|c|c|c|} \hline
Pr. & Median & Median & Median & Mean & Mean & Mean & & \\ \cline{2-7}
	& UL Accuracy & LL Accuracy & LL Calls &  UL Accuracy & LL Accuracy & LL Calls & $\frac{\mbox{Med LL Evals}}{\mbox{Med LL Calls}}$ & $\frac{\mbox{Mean LL Evals}}{\mbox{Mean LL Calls}}$ \\ \hline
TP1	&	0.000000	&	0.000000	&	206	&	0.000000	&	0.000000	&	187.43	&	73.01	&	75.31	\\	\hline
TP2	&	0.012657	&	0.000126	&	235	&	0.013338	&	0.000129	&	239.12	&	61.79	&	56.27	\\	\hline
TP3	&	0.000000	&	0.000000	&	112	&	0.000000	&	0.000000	&	93.47	&	40.00	&	42.62	\\	\hline
TP4	&	0.040089	&	0.007759	&	255	&	0.037125	&	0.007688	&	234.57	&	60.00	&	63.97	\\	\hline
TP5	&	0.008053	&	0.040063	&	203	&	0.008281	&	0.038688	&	205.81	&	77.34	&	68.50	\\	\hline
TP6	&	0.000099	&	0.000332	&	233	&	0.000097	&	0.000305	&	214.71	&	75.23	&	79.00	\\	\hline
TP7	&	0.093192	&	0.093192	&	3862	&	0.089055	&	0.095018	&	3891.08	&	69.34	&	69.08	\\	\hline
TP8	&	0.001819	&	0.000064	&	200	&	0.001968	&	0.000066	&	188.40	&	61.50	&	55.39	\\	\hline
TP9	&	0.000012	&	0.000000	&	623	&	0.000012	&	0.000000	&	598.44	&	155.09	&	155.14	\\	\hline
TP10	&	0.000103	&	0.000000	&	565	&	0.000083	&	0.000000	&	572.35	&	179.84	&	174.29	\\	\hline
\end{tabular}
\vspace{0mm}
\end{center}
\end{table*}

\begin{table*}[!hbt]
\caption{Comparison of BLEAQ against the results achieved by WJL, WLD and nested approach.} 
\vspace{-1mm}
\label{tab:sub2table3}
\begin{center}
\begin{tabular}{|c|c|c|c|c|} \hline
Pr. No.	&	\multicolumn{4}{|c|}{Mean LL Func. Evals.}\\	\cline{2-5}
	&		\multicolumn{1}{c|}{BLEAQ}	&	\multicolumn{1}{c|}{WJL (BLEAQ Savings)}	&	\multicolumn{1}{c|}{WLD (BLEAQ Savings)} & \multicolumn{1}{c|}{Nested (BLEAQ Savings)}	\\ \hline
TP1	&	14810	&	85499	(5.77)	&	86067	(5.81)	&	161204	(10.88)	\\	\hline
TP2	&	14771	&	256227	(17.35)	&	171346	(11.60)	&	242624	(16.43)	\\	\hline
TP3	&	4376	&	92526	(21.15)	&	95851	(21.91)	&	120728	(27.59)	\\	\hline
TP4	&	15285	&	291817	(19.09)	&	211937	(13.87)	&	272843	(17.85)	\\	\hline
TP5	&	15403	&	77302	(5.02)	&	69471	(4.51)	&	148148	(9.62)	\\	\hline
TP6	&	17218	&	163701	(9.51)	&	65942	(3.83)	&	181271	(10.53)	\\	\hline
TP7	&	272971	&	1074742	(3.94)	&	944105	(3.46)	&	864474	(3.17)	\\	\hline
TP8	&	12065	&	213522	(17.70)	&	182121	(15.09)	&	318575	(26.40)	\\	\hline
TP9	&	93517	&	-	-	&	352883	(3.77)	&	665244	(7.11)	\\	\hline
TP10	&	100357	&	-	-	&	463752	(4.62)	&	599434	(5.97)	\\	\hline
\end{tabular}
\end{center}
\vspace{-1mm}
\end{table*}

\section{Conclusions}
In this paper, we have proposed an efficient bilevel evolutionary algorithm (BLEAQ) which works by approximating the optimal solution mapping between the lower level optimal solutions and the upper level variables. The algorithm not only converges towards the optimal solution of the bilevel optimization problem, but also provides the optimal solution mapping at the optimum. This provides valuable information on how the lower level optimal solution changes based on changes in the upper level variable vector in the vicinity of the bilevel optimum. 

The BLEAQ approach has been tested on two sets of test problems. The first set of test problems are recently proposed SMD test problems which are scalable in terms of number of variables. The performance has been evaluated on 10-dimensional instances of these test problems. The second set of test problems are 10 standard test problems chosen from the literature. These problems are constrained and have relatively smaller number of variables. The results obtained using the BLEAQ approach has been compared against the WJL approach, the WLD approach and the nested approach. For both sets, BLEAQ offers significant improvement in the number of function evaluations at both levels. 

\section{Acknowledgments}
Ankur Sinha and Pekka Malo would like to acknowledge the support provided by Liikesivistysrahasto and Helsinki School of Economics foundation. Authors of the paper would like to acknowledge the conversations with Prof. Yu-Ping Wang.


\end{document}